\documentclass[12pt]{amsproc}

\usepackage[margin=1in]{geometry}
\usepackage[utf8]{inputenc}
\usepackage[dvipsnames]{xcolor}
\usepackage[shortlabels]{enumitem}
\usepackage{graphicx}
\usepackage{amsmath,amssymb,amsthm,mathtools}
\usepackage{MnSymbol}
\usepackage{latexsym, verbatim}
\usepackage{appendix}
\usepackage{units}
\usepackage{dsfont}
\usepackage[normalem]{ulem}
\usepackage[hidelinks]{hyperref}
\usepackage{algpseudocode}
\usepackage{relsize}
\usepackage{tikz}
\usepackage{standalone}
\usetikzlibrary{calc, positioning}
\usepackage[ruled, vlined]{algorithm2e}

\usepackage{booktabs}
\usepackage{etoolbox} 
\usepackage{xparse}  
\usepackage{bm}

\usepackage{soul}
\usepackage{nicematrix}
\usepackage{pgfplots}
\pgfplotsset{compat=1.18}
\usetikzlibrary{pgfplots.groupplots}
\usepgfplotslibrary{groupplots}
\usepgfplotslibrary{statistics}

\mathtoolsset{showonlyrefs}
\newtheorem{theorem}{Theorem}[section]
\newtheorem{corollary}[theorem]{Corollary} 
\newtheorem{lemma}[theorem]{Lemma}
\newtheorem{assumption}[theorem]{Assumption}
\newtheorem{proposition}[theorem]{Proposition}

\theoremstyle{definition}
\newtheorem{example}[theorem]{Example}
\newtheorem{definition}[theorem]{Definition}

\newtheorem*{theorem**}{Theorem\theoremnum}
\newenvironment{theorem*}[1][]{%
  \edef\theoremnum{\if\relax\detokenize{#1}\relax\else~#1\fi}
  \begin{theorem**}
}{%
  \end{theorem**}
}  

\newcommand{\be}{\bm{\epsilon}}
\newcommand{\bx}{\mathbf{x}}
\DeclareMathOperator{\pa}{pa}

\newcommand{\Gcal}{\mathcal{G}}
\DeclareMathOperator{\Cov}{Cov}

\DeclareMathOperator{\rank}{rank}
\DeclareMathOperator{\dist}{dist}
\newcommand{\RR}{\mathbb{R}}
\newcommand{\fe}{\mathbf{e}}
\newcommand{\T}{^\mathsf{T}}
\newcommand{\fb}{\mathbf{b}}
\newcommand{\fa}{\mathbf{a}}
\newcommand{\fv}{\mathbf{v}}
\newcommand{\fw}{\mathbf{w}}
\newcommand{\fu}{\mathbf{u}}

\title{Tensor-based second-order causal discovery}
\author{Nathan Ouyang, Kexin Wang and Anna Seigal}
\date{}

\begin{document}

\begin{abstract}
Causal discovery seeks to uncover the causal dependencies among variables. For this purpose, we propose an algorithm called Tensor-based Second-order Causal Discovery (TSCD). Its input is a tensor obtained from the covariance matrices of observational and interventional data. Assuming the causal dependencies follow a linear structural equation model on a directed acyclic graph (DAG), TSCD outputs the DAG and the functions on its edges, requiring only that the noise variables are uncorrelated. We also implement a version of the approach for nonlinear models. Our focus on second-order statistics (via the covariance matrices) is motivated by their statistical and computational efficiency relative to higher-order moments, their identifiability relative to first-order statistics, and that they work regardless of whether the variables are Gaussian. We show that TSCD has identifiable causal order and parameters from a number of interventions that is logarithmic in the number of variables. Experiments show that TSCD is robust to noise, competitive with existing methods, and scales to hundreds of variables.
\end{abstract}

\maketitle

\section{Introduction}

Causal discovery is a fundamental problem in biology~\cite{friedman2000using,sachs2005causal}, neuroscience \cite{siddiqi2022causal}, economics~\cite{imbens2020potential}, public health~\cite{greenland1999causal} and machine learning~\cite{scholkopf2021toward}. Its goal is to determine the causal dependencies among a set of variables and to quantify the strength of the causal effects. With this information, one seeks to understand the system mechanisms, make predictions and answer counterfactual questions.

Structural equation models (SEMs)~\cite{pearl2009causality,wright1934method,peters2017elements} are a framework for describing causal relationships. In a SEM, each variable is a function of its direct causes, together with a noise variable. The causal structure is often represented by a directed acyclic graph (DAG) whose edges encode direct causal effects. Linear SEMs (LSEMs) are SEMs in which the functions on the edges are linear. We study both linear and nonlinear SEMs.

Interventions are crucial in causal discovery because observational data alone in general identifies only a set of DAGs that induce the same conditional independence relations, known as a Markov equivalence class~\cite{verma2022equivalence}. Perturbing variables alters their dependence on their parents, breaking observational symmetries.  An example is Perturb-seq~\cite{dixit2016perturb}, where genetic perturbations are applied and their effects on gene expression are measured. 

In this paper, we propose a new algorithm for causal discovery, which we call \emph{Tensor-based Second-order Causal Discovery} (TSCD). It recovers an unknown causal order, DAG, and model parameters, from the covariance matrices of observational and interventional data, assuming that the data follows an LSEM on an unknown DAG. The only distributional assumption required is that the noise variables are uncorrelated. We also propose an extension of the method to nonlinear SEMs, which we call TSCD-nonlinear.

TSCD is {\em second-order} because it takes as input only the covariance matrices of the data.
A method that uses only second-order statistics is necessary for a system with multiple Gaussian noise variables. 
It is practical when noise variables are close to Gaussian, making it difficult to estimate
higher-order statistics. Second-order statistics are practical in applications where per-context sample sizes are too small to reliably estimate full distributions. We see that such scenarios limit the accuracy of approaches that use
non-Gaussianity assumptions, such as Linear Non-Gaussian Acyclic Model (LiNGAM)~\cite{shimizu2006lingam}, or
distributions for conditional independence testing, such as the Peter-Clark (PC) algorithm~\cite{spirtes2000causation}.
Second-order statistics are also useful summaries for situations where raw data are unavailable due to privacy, storage, or communication constraints. They are more informative than first-order statistics such as 
mean vectors, from which the parameters in our model are not identifiable.

Several works have used second-order statistics for 
causal discovery and parameter estimation in an LSEM. From observational data alone, the problem is well-known to be non-identifiable~\cite{peters2017elements}. To get around this, approaches include 
restricting the noise variances~\cite{peters2014identifiability}, applying nonlinear transformations~\cite{schultheiss2023ancestor}, parameter identifiability for a fixed graph~\cite{foygel2012half,drton2011global}, and outputting an equivalence class of DAGs~\cite{chickering2002optimal}. By comparison, TSCD makes no assumptions on the graph, beyond that it is a DAG, and makes no extra assumptions on parameters, beyond those of a usual LSEM. Its only distributional assumption on the noise variables is that they are uncorrelated. 

To identify the causal order and model parameters, TSCD builds on a line of work to use 
the compatibility of causal structure across multiple contexts (also called views or environments) to recover parameters from a collection of covariance matrices~\cite{heurtebise2025identifiable,rothenhausler2015backshift,peters2016causal,squires2023linear}. It studies observational data together with data generated via perfect interventions. 

TSCD is \emph{tensor-based} because it works by constructing a tensor from the covariance matrices, whose decomposition encodes the causal relationships and intervention patterns. For $p$ variables in $k$ contexts, it is a $p \times p \times k$ tensor that is symmetric in its first two indices; i.e., a stack of $k$ many $p \times p$ symmetric matrices, see Figure~\ref{fig:precision tensor construction}. The stability of second-order statistics across contexts is a necessary condition to be a root~\cite{peters2017elements}. We use the tensor decomposition to devise a subspace-membership condition that is both necessary and sufficient. Subspace membership has well-studied connections to tensor decomposition, see~\cite{de2006link,johnston2023computing,kileel2025subspace,landsberg2011tensors,ranestad2026real,wang2023lower,wang2025multi}. TSCD identifies candidate root nodes using a subspace-proximity criterion, refined using pairwise intervention-asymmetric correlation tests. The tensor formulation yields identifiability guarantees, on the causal order and consequently on the parameters of the model. We show that \(\lceil \log_2 p\rceil\) interventional environments, and one observational environment, are sufficient and, in the worst case, necessary to recover the order and parameters.

The paper is organized as follows. In Section~\ref{sec:setup}, we give background, identifiability results, and related work. In Section~\ref{sec:precision}, we show how second-order statistics across observational and intervention contexts are combined into a tensor, which we call the precision tensor. 
In Section~\ref{sec:identifiability} we prove our identifiability results.
In Section~\ref{sec:tensor-based}, we explain how roots can be recovered.
In Section~\ref{sec:algorithm}, we present the TSCD algorithm. 
Section~\ref{sec:experiments} evaluates TSCD on linear and nonlinear models and real datasets.

\section{Setup and main results}\label{sec:setup}

\subsection{Directed acyclic graphs and structural equation models}

Causal relationships among a set of variables can be represented by a directed graph, where nodes are variables and directed edges encode the direct causal relationships between them.

\begin{definition}
A {\em directed acyclic graph} (DAG) is a graph $\Gcal = (V, E)$ with nodes $V$ and directed edges $E \subseteq V \times V$, such that $\Gcal$ contains no directed cycles. 
If there is an edge $j \to i$, we say $j$ is a {\em parent} of $i$. We denote the set of parents of node~$i$ by $\pa(i)$. A {\em root} is a node with no parents. If there is a directed path $j \to \cdots \to i$, we say that $j$ is an {\em ancestor} of $i$.
\end{definition}

A structural equation model writes each variable as a function of its parents in the graph together with an exogenous noise variable. We consider DAGs with vertex set $[p] = \{ 1, \ldots, p\}$.

\begin{definition}
Given a DAG with vertex set $[p]$, where node $i$ represents variable $x_i$, a
{\em structural equation model} (SEM) writes 
\[
x_i = f_i\big(x_{\pa(i)}, \epsilon_i \big),\quad \text{for all} \quad i = 1,\dots,p,
\]
where $f_i$ is a function, $x_{\pa(i)}$ is the set of variables with direct edges to $x_i$, 
and $\epsilon_i$ is an exogenous noise variable. 
A {\em linear structural equation model} (LSEM) is an SEM in which each function $f_i$ is linear:
$$
x_i = \sum_{j = 1}^p \Lambda_{ij} x_j + \epsilon_i,\qquad i = 1,\dots,p,
$$
with $\Lambda_{ij}$ the weight of the dependency $x_j \to x_i$
 (or 0 if there is no such edge). We say that an LSEM follows the structure of a DAG if $\Lambda_{ij} \neq 0$ if and only if $j \in \pa(i)$.
\end{definition}

Writing $\bx = (x_1,\ldots,x_p)\T$ and 
$\be = (\epsilon_1,
\ldots,\epsilon_p)\T$, 
the LSEM can be written as $\bx = \Lambda \bx + \be$ or 
\begin{equation}\label{eq: lsem}
\bx = (I-\Lambda)^{-1} \be,
\end{equation}
where $\Lambda \in \mathbb{R}^{p \times p}$ is the matrix of edges weights. Denote the covariance of exogenous noise variables by \(\Omega\).
Let \(\bx \) have covariance matrix \(\Sigma\) and \emph{precision matrix} 
\(
\Theta := \Sigma^{-1}.
\)
Then
\begin{equation}
\label{eqn:sigma-theta}
\Sigma = (I-\Lambda)^{-1} \Omega (I-\Lambda)^{-\mathsf{T}} \quad \text{and}
\quad 
\Theta = (I-\Lambda)^{\mathsf{T}} \Omega^{-1} (I-\Lambda),
\end{equation}
see e.g.~\cite{brito2002new,sullivant2023algebraic}.
The noise variables being uncorrelated implies that \(\Omega \) is diagonal. 
The fact that the LSEM follows the structure of a DAG means there is a permutation of indices that makes $\Lambda$ strictly lower triangular. This ensures that $(I-\Lambda)$ is invertible.
Finding the reordering that makes $\Lambda$ lower triangular is the causal discovery problem.

The uniqueness of causal order and recovered parameters in the model is important for interpretability and downstream analysis. 
Given a \emph{known} causal order,
the parameters in the model are identifiable from the covariance $\Sigma$, 
e.g. using the LDL decomposition. However, the  matrix \(\Lambda\) in an LSEM on an \emph{unknown} causal order is not identifiable from $\Sigma$. This can be seen by observing that an LDL decomposition of $P \Sigma P\T$ exists for every permutation matrix $P$.
Recovery of the causal order is sufficient for recovery of the DAG and model parameters: subsequent steps, e.g. via regression, can detect the weights on the edges in the total DAG on a causal order, and hence whether edges are present or absent. 

\begin{example}\label{ex: non-identifiable}
Consider the DAG \(x_1 \to x_2\) with SEM
\[
x_1 = \epsilon_1, \qquad x_2 = 2x_1 + \epsilon_2,
\quad 
\text{where} \quad \Omega = \begin{pmatrix} 1 & 0 \\ 0 & 1 \end{pmatrix}.\]
Now consider the reversed DAG \(x_2 \to x_1\), with SEM
\[
x_2 = \tilde{\epsilon}_2, \qquad x_1 = \tfrac{2}{5}x_2 + \tilde{\epsilon}_1, \quad \text{where} \quad 
\tilde{\Omega} = \begin{pmatrix} \frac15 & 0 \\ 0 & 5 \end{pmatrix}.\]
Both models have covariance matrix 
\(
\Sigma =
\begin{pmatrix}
1 & 2\\
2 & 5
\end{pmatrix}
\).
\end{example}

\subsection{Interventions}

Interventions can recover the parameters in LSEMs from covariance matrices identifiably.
We consider \emph{perfect interventions}~\cite{eberhardt2012number}, which remove the influence of parents on the intervened variables.

\begin{definition}[Perfect intervention]
Fix an LSEM
\(
\bx = (I-\Lambda)^{-1} \be,
\)
where \(\be\) is a vector of noise variables and $\Lambda_{ij}$ is the weight on the edge $x_j \to x_i$. Let \(I \subseteq [p]\) be a set of intervened nodes. A \emph{perfect intervention} on \(I\) removes all incoming edges to nodes in \(I\), i.e., sets 
\[ \Lambda_{ij} = 0 \quad \text{ for all } \quad  i \in I \quad \text{and} \quad  j \in \pa(i), \] 
and changes \(\epsilon_i\) to \(\epsilon_i'\) for \(i \in I\),
such that the $p$ noise variables remain mutually uncorrelated.
Hence
the intervention at $I$ can change $\Omega$ at positions $(i,i)$ for $i \in I$. 
We perform multiple interventions on a single LSEM, and call each interventional setting an \emph{intervention context}. The set of intervened nodes in each intervention context is assumed to be known.
\end{definition}

\subsection{Identifiability}\label{sec:theory}

We present the identifiability results for our model. The input consists of the covariance matrices in the observational context and multiple intervention contexts.

\begin{assumption}\label{ass:lsem}
The vector of variables \(\bx=(x_1,\dots,x_p)\T\ \)
follows an unknown LSEM
\[
\bx = \Lambda \bx + \be,
\]
on an unknown DAG $\mathcal{G}$,
with noise variables \(\be=(\epsilon_1,\dots,\epsilon_p)\T\) mutually uncorrelated with nonzero variances, where \(\Lambda_{ij} \) is the edge weight on the edge $x_j \to x_i$.
In the intervention contexts, perfect interventions are performed at a known nonempty subset of variables.
We have access to the covariance matrix $\Sigma_1$ of data in the observational context and covariance matrices $\Sigma_2, \ldots, \Sigma_k$ of data under intervention contexts. 
\end{assumption}

Under Assumption~\ref{ass:lsem}, we show that second-order statistics suffice to recover the causal structure from one observational and $\lceil \log_2 p\rceil$ interventional contexts.

\begin{theorem}[Identifiability from intervention patterns]\label{thm: identifiability}
Consider an unknown LSEM
\[
\mathbf x = \Lambda \mathbf x + \boldsymbol{\epsilon},
\]
under Assumption~\ref{ass:lsem}.
For each \(x_i\), define its intervention-pattern vector \(\mathbf b_i\in\{0,1\}^k\) by
\[ (\mathbf b_i)_j= \begin{cases} 0 & x_i \text{ intervened in context } j \\  1 &  \text{otherwise.} \end{cases} \] 
If the vectors \(\mathbf b_1,\dots,\mathbf b_p\) are nonzero and pairwise distinct, 
then the causal order and adjacency matrix \(\Lambda\) are identifiable.
\end{theorem}

We prove Theorem~\ref{thm: identifiability}
by forming a tensor from the covariance matrices under observational and intervention contexts. Its decomposition encodes the rows of \(I-\Lambda\) and the intervention-pattern vectors. 
Identifiability follows from the uniqueness of its decomposition. 
The proof is in Section~\hyperref[proof: thm: identifiability]{\ref*{sec:tensor-based}}.

\begin{corollary}\label{cor: log2pcontexts}
Given an observational context, \(\lceil \log_2 p\rceil\) perfect-intervention contexts are sufficient and in the worst case necessary to recover the DAG and adjacency matrix~$\Lambda$.
\end{corollary}

Corollary~\ref{cor: log2pcontexts} follows by observing that \(\lceil \log_2 p\rceil\) intervention contexts suffice to assign distinct intervention-pattern vectors to all nodes. With fewer contexts, there exist two variables with the same intervention-pattern vectors, making the edge between them non-identifiable. See Section~\hyperref[proof: cor: log2pcontexts]{\ref*{sec:tensor-based}}.

\subsection{Algorithmic overview}

The proof of Theorem~\ref{thm: identifiability} suggests an algorithm for recovering the causal order and parameters of the model: stack the covariance matrices across the contexts into a tensor and decompose it.

TSCD performs this tensor decomposition, with two modifications. First, modified precision (inverse covariance) matrices are stacked across contexts, rather than the covariance matrices, since they yield a lower rank decomposition that encodes the DAG and intervention patterns. Second, the algorithm does not compute a full decomposition, 
since this is less numerically stable. Instead, it uses the sparsity of $\Lambda$ and the known intervention-pattern vectors to build the decomposition step by step.

TSCD uses a subspace-proximity score that, in the population setting, takes the value 1 only at roots of the DAG. 
  In finite samples, it keeps a few high-scoring candidates and uses correlations to choose among them. After selecting a root, it removes it and repeats.

\subsection{Related Work}\label{sec:relatedworks}

The PC algorithm~\cite{spirtes2000causation} is a constraint-based method that recovers the Markov equivalence class of a DAG using conditional independence tests. It iteratively removes edges based on conditional independence relations and orients edges using logical rules.

Greedy equivalence search (GES)~\cite{chickering2002optimal} learns a DAG by optimizing a score (e.g., the Bayesian Information Criterion) computed over the space of equivalence classes using a greedy forward-backward procedure. Its extension to interventional data, greedy interventional equivalence search (GIES)~\cite{hauser2012characterization}, incorporates interventions into the scoring framework to improve identifiability. 

Interventional greedy sparse permutation (IGSP)~\cite{wang2017permutation}  learns a causal DAG from observational and interventional data with unknown intervention targets by searching over permutations and selecting the sparsest graph consistent with conditional independence and invariance constraints across environments. It is nonparametric but relies on accurate conditional independence testing.

The linear non-Gaussian acyclic model (LiNGAM) via Independent Component Analysis (ICA) \cite{shimizu2006lingam} identifies a linear SEM by assuming non-Gaussian noise. It estimates the mixing matrix via ICA and recovers the causal order by permuting and scaling components to match a DAG structure. Under the assumptions that the noise variables are independent and non-Gaussian, ICA-LiNGAM is identifiable. It breaks down when some noise variables are Gaussian or dependent.

DirectLiNGAM~\cite{shimizu2011directlingam} improves upon ICA-LiNGAM by estimating a causal order through iterative regression: at each step, it identifies an exogenous variable by testing independence between residuals and candidate predictors, removes its effect, and repeats. Though more stable than ICA-based methods, it still relies on independence testing and non-Gaussianity.

NOTEARS~\cite{zheng2018dags} learns a DAG by solving a continuous optimization with a differentiable acyclicity constraint.  It assumes access to individual samples rather than only second-order statistics
and relies on optimization over a nonconvex objective.

Invariant causal prediction (ICP)~\cite{peters2016causal} identifies causal parents of a variable by searching over subsets of variables and selecting those for which the conditional distribution of the target is invariant across environments. This exhaustive subset search is computationally expensive and requires that there is at least one intervention on each parent. We do not compare to this method in our synthetic experiments because existing implementations are not publicly available in Python and the method is designed for parent identification of a single target variable rather than full DAG recovery.

The paper \cite{reisach2021beware} shows that commonly used data-generating processes can introduce artifacts that make the causal learning problem artificially easy. Heuristics such as sorting variables by their marginal variances can recover the causal order well under synthetic settings, not because they capture causal structure, but because the data generation encodes ordering information. 
To account for this, we include a sort-and-regress baseline (SortRegress) in the experiments, which orders variables by their marginal variances and then performs regression. This checks that our method does not rely on such artifacts.

Our approach differs from the above methods in several ways. Unlike constraint-based methods such as PC and IGSP, it does not require access to individual samples and does not use conditional independence testing. Unlike ICA-LiNGAM and DirectLiNGAM, we only require that the noise variables are uncorrelated and do not assume independent or non-Gaussian noise variables. 
Our method is more computationally efficient than GES/GIES for many variables, as the search space becomes large for a score-based greedy search. 
It is also more computationally efficient than NOTEARS for many variables, since the constrained optimization calculates an expensive matrix exponential.
Unlike invariance-based approaches such as ICP, our method does not require exhaustive subset search or interventions that target each parent individually.

\section{The precision tensor}
\label{sec:precision}

We construct the tensor that is the core of our identifiability proof and algorithm. 
It is obtained by combining precision matrices across contexts. 
For $k$ contexts and $p$ variables, it is a tensor $T \in \RR^{p \times p \times k}$ that is symmetric in its first two factors. 
 Its decomposition encodes the DAG and the intervention-pattern vectors.
For related ideas based on stacking covariance matrices rather than precision matrices, see~\cite{wang2026multi}.

We stack precision matrices instead of covariance matrices because they lead to a decomposition involving the rows of $I-\Lambda$ rather than $(I - \Lambda)^{-1}$. The former stays the same on the rows indexed by variables that are not intervened, but the latter does not. 

\begin{definition}[Outer product]
For $\fu \in \RR^p$, the outer product $\fu \otimes \fu$ is the matrix $\fu \fu\T$, which has $(i,j)$ entry $\fu_i \fu_j$. 
For \(\fu \in \mathbb{R}^p\) and \(\fv \in \mathbb{R}^k\), the outer product \(\fu \otimes \fu \otimes \fv \in \mathbb{R}^{p \times p \times k}\) is 
\[
(\fu \otimes \fu \otimes \fv)_{i j \ell} = \fu_i \fu_j \fv_\ell.
\]
\end{definition}

Given a tensor $T \in \RR^{p \times p \times k}$, we denote its entry at position $(i,j,\ell)$ by $T_{ij\ell}$ or $T(i,j,\ell)$. 
We say that a tensor $T$ is symmetric under swapping the first two indices, if
\(
T_{ij\ell}=T_{ji\ell}.
\) for all \( i, j \in [p] \) and \( \ell \in [k] \). 
For example, this holds for the tensor $\fu \otimes \fu \otimes \fv$. 

\begin{definition}[Tensor slices]
For a tensor \(T\in\RR^{p\times p\times k}\), 
fixing one index produces a matrix called a \emph{slice} of the tensor. We denote them by $T(i,:,:) \in \RR^{p \times k}$, $T(:,j,:) \in \RR^{p \times k}$, and $T(:,:,\ell) \in \RR^{p \times p}$.
For example, $T(i,:,:)$ is the matrix with $(j,\ell)$ entry $T(i,j,\ell)$. 
If $T$ is symmetric under swapping the first two indices, then \(T(i,:,:)\) and \(T(:,i,:)\) are the same.
\end{definition}

For more on tensors and their slices, see e.g.~\cite{kolda2009tensor}.
We stack the decompositions in~\eqref{eqn:sigma-theta} together across contexts. To do so, we begin with the following observation, recorded here as a proposition for clarity, which restates the expression for the precision matrix.
Let \(\mathbf{e}_i \in \RR^p\) denote the \(i \)-th standard basis vector.

\begin{proposition}\label{prop:rank-one-sums-observational}
Let $\Theta$ be the precision matrix from the LSEM $\bx = (I - \Lambda)^{-1} \be$ on a DAG, where the noise covariance matrix $\Omega$ is diagonal with non-zero entries. 
    Let \(\mathbf{v}_i\) denote the \(i\)-th row of \((I-\Lambda)\), and let \(\omega_i\) be the \(i\)-th diagonal entry of \(\Omega^{-1}\). 
    Then 
    \begin{equation}
        \label{eqn:theta-as-sum}
    \Theta 
    =
    \sum_{i =1 }^p \omega_i \, \mathbf{v}_i^{\otimes 2}.
    \end{equation}
    Node $i$ is a root of the DAG if and only if $\fv_i = \fe_i$.
\end{proposition}

\begin{proof}
The last statement follows because $\Lambda_{ij} = 0$ for all $j \in [p]$ exactly when $i$ is a root.
    The matrix $\Theta$ has decomposition $(I - \Lambda)\T \Omega^{-1} (I - \Lambda)$, see~\eqref{eqn:sigma-theta}. Its $(j,\ell)$ entry is $\sum_{i = 1}^p \omega_i (I - \Lambda)_{i j} (I - \Lambda)_{i \ell}$. The product $(I - \Lambda)_{i j} (I - \Lambda)_{i \ell}$ is the $(j,\ell)$ entry of $\fv_i^{\otimes 2}$. 
\end{proof}

We derive a related expression for the precision matrix under an intervention context. 
The nodes that are intervened become roots, as the effect of their parents is set to zero. This zeros out the corresponding rows of $\Lambda$. The rows of $\Lambda$ that are not intervened stay the same.

\begin{proposition}\label{prop:rank-one-sums}
Under the same setting as Proposition~\ref{prop:rank-one-sums-observational}, 
let $\Theta_j$ be the precision matrix under a perfect intervention on $I_j \subseteq [p]$. 
Then 
    \[
    \Theta_j
    =
    \sum_{i \in I_j^c} \omega_i \, \mathbf{v}_i^{\otimes 2}
    +
    \sum_{i \in I_j} \omega_{i}' \, \mathbf{e}_i^{\otimes 2},
    \]
    where \(\omega_i'\) are new noise parameters.
\end{proposition}
\begin{proof} The precision matrix can be decomposed as
    \[
    \Theta_j = (I-\Lambda_j)^{\mathsf{T}} \Omega^{-1}_j (I-\Lambda_j) = \sum_{i = 1}^p (\Omega_j^{-1})_{ii} (I - \Lambda_j)_{i, :}^{\otimes 2} ,
    \]
 as in Proposition~\ref{prop:rank-one-sums-observational}. 
 It remains to relate the matrices $\Lambda_j$ and $\Omega_j$ to $\Lambda$ and $\Omega$ respectively. 
    Under a perfect intervention on node \(\ell \), all incoming edges to \( \ell \) are removed, i.e., the entries of the \( \ell \)-th row of \(\Lambda\) are zeroed out.
    Thus $\Lambda_j$ is zero on rows $i \in I_j$ and equals $\Lambda$ on rows $i \notin I_j$. Finally,
    \( 
    (\Omega_j^{-1})_{ii} \) is $
    \omega_{i}'$ for $i \in I_j$ and $\omega_i$ for $i \notin I_j$, since the intervention can alter the variance of the exogenous noise variables at intervened variables.
\end{proof}

Though the above matrix decompositions are not identifiable, the recovery of certain parameters is unique: the coefficients of the vectors $\fe_i$.

\begin{proposition}[Rank reduction]
\label{prop:rank-reduction}
Let intervention context $j \geq 2$ be a perfect intervention on a nonempty set $I_j \subseteq [p]$. 
    For each \(i \in I_j\), there is a unique scalar \(\mu \) such that
\[
\rank \left( \Theta_j - \mu \, \mathbf{e}_i^{\otimes 2}\right)
=
\rank \left(\Theta_j \right) - 1.
\]
Moreover, we have $\mu = \tfrac{1}{(\Sigma_j)_{ii}}$.
\end{proposition}

\begin{proof}
The matrix
\(
\Theta_j-\mu \fe_i\fe_i\T
\)
drops rank by one if and only if
\[
\mu=\frac{1}{\fe_i\T\Theta_j^{-1}\fe_i},
\]
 by Wedderburn rank reduction~\cite{wedderburn1934lectures}. 
The denominator of $\mu$ is 
$\fe_i\T\Sigma_j\fe_i = (\Sigma_j)_{ii}$.
\end{proof}

\begin{corollary}
Under Assumption~\ref{ass:lsem}, we can recover the set of matrices
\[
\mathcal{M}
=
\left\{
\sum_{i \in I_j^c} \omega_i \, \mathbf{v}_i^{\otimes 2}
:\ j=1,\dots,k
\right\}.
\]    
\end{corollary}

\begin{proof}
    The precision matrix $\Theta_j$ is a sum of two types of term: those of the form $\omega_i \fv_i^{\otimes 2}$, for $i \notin I_j$, and those of the form $\mu \fe_i^{\otimes 2}$ where $\mu$ is unique to $(j, i)$. 
    The coefficients of $\fe_i^{\otimes 2}$ can be recovered by Proposition~\ref{prop:rank-reduction}, thus the terms of the second type can be removed.
\end{proof}

\begin{definition}[Precision tensor] \label{def: precision tensor}
We stack the matrices in \(\mathcal{M}\) into a tensor \(T \in \mathbb{R}^{p \times p \times k}\) with slices
\(
T(:,:,j) = \sum_{i \in I_j^c} \omega_i \, \mathbf{v}_i^{\otimes 2}.
\)
We call \( T\) the \emph{precision tensor}.
It has decomposition
\begin{equation}\label{eq: T}
T
=
\sum_{i=1}^p
\omega_i \, \mathbf{v}_i \otimes \mathbf{v}_i \otimes \mathbf{b}_i,
\end{equation}
where \(\mathbf{b}_i \in \mathbb{R}^k\) has \((\fb_i)_j=1\) if node \(i\) is not intervened in context \(j\), and \(0\) otherwise. We call $\fb_i$ the \emph{intervention-pattern vector} of the $i$-th variable.
We collect the intervention-pattern vectors $\fb_i$ together to form the \emph{intervention-pattern matrix} $B \in \RR^{p \times k}$ with $(i,j)$ entry $(\fb_i)_j$, which is $0$ if variable $i$ is intervened in context $j$ and $1$ otherwise.
\end{definition}


\begin{figure}[htbp]
  \centering
  \includegraphics[width=0.6\textwidth]{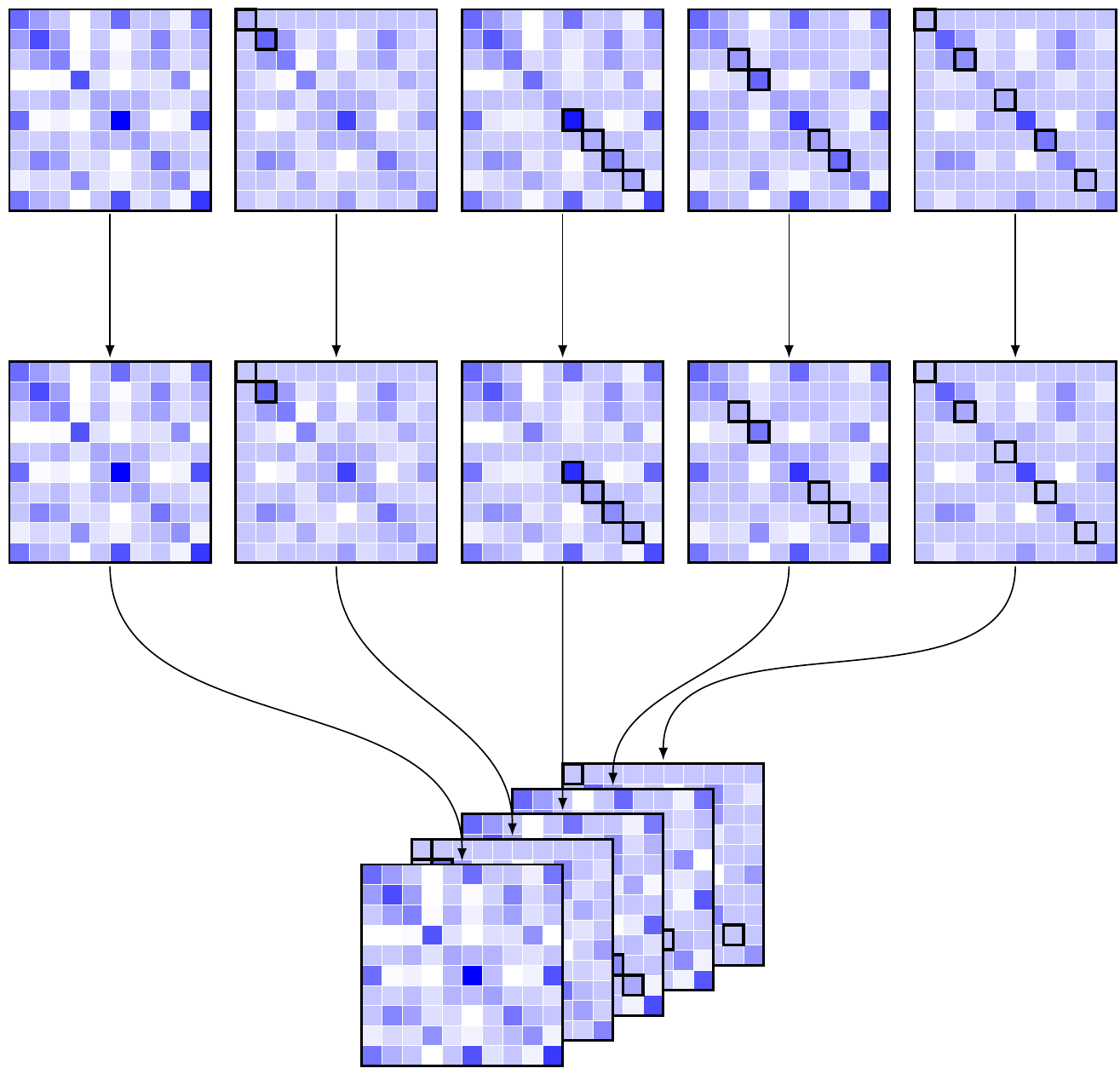}
  \caption{Constructing the precision tensor for 10 variables and five contexts (one observational). Top row: Precision matrices for each context. Middle row: The matrices after rank reduction. Bottom row: The matrices are stacked to form the precision tensor. The boxed entries are the only ones that change in the rank reduction step; they correspond to the zeros of the intervention-pattern matrix $B \in \{ 0, 1\}^{10 \times 5}$ with rows $\fb_i\T$. For example, the boxed entries in the fifth matrix reflect that column $5$ of $B$ has zeros in entries $i = 1,3,5,7,9$.
  }
  \label{fig:precision tensor construction}
\end{figure}

\section{Identifiability of the precision tensor decomposition}
\label{sec:identifiability}

The decompositions of the individual precision matrices~\eqref{eqn:theta-as-sum} are not unique: multiple non-equivalent choices of vectors $\fv_i$ and scalars $\omega_i$ yield the same matrix $\Theta$, cf. Example~\ref{ex: non-identifiable}. 
In comparison, the decomposition of the precision tensor~\eqref{eq: T} is unique, under mild conditions. This typifies the benefit of combining second-order statistics across contexts.

The decomposition of the precision tensor $T$  is a CP decomposition: it writes a tensor as a sum of outer products. The decomposition is said to be unique if any such expression has summands that coincide with $\fv_i\otimes \fv_i\otimes \fb_i$, up to scale and permutation~\cite{chiantini2012generic}. Once the terms $\fv_i \otimes \fv_i \otimes \fb_i$ are recovered, the vectors $\fv_i$ and $\fb_i$ are obtained uniquely, up to scale.

\begin{lemma}\label{lem: unique decomposition}
The decomposition of $T$ is unique when $\fb_1,\ldots,\fb_p$ are pairwise distinct.
\end{lemma}
\begin{proof}
The decomposition of $T$ is unique when no pair of vectors $\fb_i,\fb_j$ are collinear and $\fv_1,\ldots,\fv_p$ are linearly independent, by the Kruskal criterion~\cite{kruskal1977three}. 
Linear independence of $\{\fv_i\}_{i=1}^p$ follows from the invertibility of $I-\Lambda$.
Since each $\fb_i$ is binary and nonzero, collinearity only occurs if the vectors coincide.
\end{proof}

Theorem~\ref{thm: identifiability} follows by combining Lemma~\ref{lem: unique decomposition} with the structure of \(I-\Lambda\).

\begin{proof}[Proof of Theorem \ref{thm: identifiability}] \label{proof: thm: identifiability}
Consider the tensor decomposition~\eqref{eq: T}. The vectors $\fv_i$ are linearly independent since they are rows of $I-\Lambda$, which is invertible. 
By Lemma \ref{lem: unique decomposition}, the decomposition of $T$ is identifiable. So, $I-\Lambda$ is identifiable up to row scaling and permutation.
The ambiguity is resolved since each $\fv_i$ is the $i$-th row of $I-\Lambda$, whose $i$-th entry is one.
\end{proof}

The resolution of the scaling and permutation ambiguity in the proof of Theorem~\ref{thm: identifiability} also appears in LiNGAM~\cite{shimizu2006lingam}. Corollary~\ref{cor: log2pcontexts} quantifies how many intervention contexts are sufficient and in the worst case necessary for unique decomposition.

\begin{proof}[Proof of Corollary \ref{cor: log2pcontexts}] \label{proof: cor: log2pcontexts}
With one observational context, each intervention-pattern vector \(\mathbf b_i \in \{0,1\}^k\) has first entry equal to \(1\). Hence, with \(k-1\) additional intervention contexts, there are at most \(2^{k-1}\) distinct binary patterns available. To ensure that all \(p\) nodes have distinct intervention patterns, it suffices that
\[
2^{k-1} \ge p,
\]
which holds when \(k-1 \ge \lceil \log_2 p \rceil\). In this case, we can assign distinct binary vectors \(\mathbf b_i\) to each node, and define the intervention sets \(I_j\) by intervening on node \(i\) in context \(j\) whenever \((\mathbf b_i)_j = 0\). By Theorem~\ref{thm: identifiability}, this ensures identifiability of \(\Lambda\).

For necessity, assume we have fewer than $\lceil \log_2 p \rceil$ additional interventions, then there are two nodes, say 1 and 2, that share the same intervention pattern.
Their contributions to each slice $T(:,:,i)$ of the tensor $T$ in \eqref{eq: T} appear only through
\[
M = \omega_1\mathbf v_1^{\otimes 2} + \omega_2\mathbf v_2^{\otimes 2}.
\]
We show that this decomposition is not unique. 
Suppose $M$ is only nonzero in the top-left \(2\times 2\) submatrix
\[
M_{[2]} =
\begin{pmatrix}
a & b \\
b & c
\end{pmatrix}.
\]
Then there exist distinct pairs \((\mathbf v_1,\mathbf v_2)\) that produce the same \(M_{[2]}\). For example,
\[
\mathbf v_1 = (0,1,0,\ldots,0),\ \mathbf v_2 = (a,b,0,\ldots,0),
\quad \text{and} \quad 
\mathbf v_1 = (1,0,0,\ldots,0),\ \mathbf v_2 = (b,c,0,\ldots,0).
\]
Hence, the edge direction between the first two nodes is not identifiable.
\end{proof}

\section{Tensor-based root selection}
\label{sec:tensor-based}

The previous section demonstrates the identifiability of the causal order, and parameters, from the precision tensor. 
In principle, the parameters of the model can be recovered via usual tensor decomposition of~\eqref{eq: T}, e.g. using~\cite{wang2025multi,sorber2015structured}.
However, in practice such tensor decomposition suffers from numerical instability, worsened by the fact that the intervention-pattern vectors are often close to parallel, as most variables are not intervened. 
To avoid this, TSCD does not compute a full decomposition. Instead, it uses the sparsity of $\Lambda$ and the known intervention-pattern vectors to build the decomposition.

We show how root nodes can be identified from the precision tensor.
Recall that we denote by $T(i,:,:) \in \RR^{p \times k}$ the slice of $T$ obtained by fixing the first index to take the value $i$.  The slices $T(1,:,:), \ldots, T(p,:,:)$  span a linear space, which we denote by $A$. The form of the decomposition in~\eqref{eq: T}
implies that \[
A = \mathrm{span}\{ \mathbf v_i \otimes \mathbf b_i : i=1,\dots,p \} \subseteq \mathbb{R}^{p \times k}.
\]
The connection between uniqueness of tensor decomposition and of the rank one matrices or tensors in a linear space is well-studied, see e.g.~\cite{landsberg2011tensors}.
In our case, it is as follows. 

\begin{lemma}
The only rank-one matrices in $A$ are $\fv_i\otimes \fb_i$ up to scaling for $i=1,\ldots,p$.
\end{lemma}
\begin{proof}
Assume for contradiction that there exists a rank-one matrix $\fa\otimes \fb$ in $A$ that is not collinear to any  $\fv_i\otimes \fb_i.$ Then $\fa\otimes \fb = \sum_{j=1}^p \mu_j \fv_j\otimes \fb_j$, for some $\mu_j\in \RR$. Without loss of generality, we assume $\mu_1\neq 0$. Hence we can build a new decomposition of $T$ that uses 
$\fa\otimes \fb$ instead of $\fv_1\otimes \fb_1$.
That is, there exists $\mathbf{c}_i\in \RR^p$ for $i=1,\ldots,p$ such that $T = \mathbf{c}_1\otimes \fa\otimes \fb + \sum_{j=2}^p \mathbf{c}_j\otimes \fv_j\otimes \fb_j$. The uniqueness result in Lemma~\ref{lem: unique decomposition} was stated for a decomposition in which the two vectors of length $p$ are the same (a partially symmetric decomposition), but the decomposition is still unique if all vectors are allowed to differ. The new decomposition of \(T\) is thus a contradiction. 
\end{proof}

\begin{lemma}\label{lem: proj_norm}
Node $i$ is a root if and only if 
$\fe_i\otimes \fb_i\in A$.
\end{lemma}

\begin{proof}
Node $i$ is a root if and only if it has no incoming edges. This implies $\fv_i = \fe_i$, so $\fe_i\otimes \fb_i \in A$. 
Conversely, suppose \(\fe_i \otimes \fb_i \in A\). Then there exist scalars \(\{\alpha_j\}_{j=1}^p\) such that
\begin{equation}\label{eq: ei tensor bi}
\fe_i \otimes \fb_i = \sum_{j=1}^p \alpha_j \fv_j \otimes \fb_j.
\end{equation}
The vectors \(\{\fv_j\}_{j=1}^p\) are linearly independent, since they are rows of the invertible matrix $I-\Lambda$.
Thus, there exists a dual set of vectors \(\{\fw_\ell\}_{\ell=1}^p\) such that
$\langle \fw_\ell, \fv_j \rangle = \delta_{\ell j}.$
Multiplying both sides of \eqref{eq: ei tensor bi} by \(\fw_\ell\T\) yields
\[
\langle \fw_\ell, \fe_i \rangle \fb_i = \alpha_\ell \fb_\ell.
\]
For \(\ell \neq i\), the vectors \(\fb_\ell\) and \(\fb_i\) are not collinear, so the above equality implies
\[
\langle \fw_\ell, \fe_i \rangle = 0 \quad \text{and} \quad \alpha_\ell = 0.
\]
Thus only the term \(\ell = i\) remains, and we obtain
\[
\fe_i \otimes \fb_i = \alpha_i \fv_i \otimes \fb_i.
\]
Since \(\fb_i \neq 0\), this implies \(\fe_i = \alpha_i \fv_i\). Comparing the \(i\)-th entries (both equal to \(1\)) gives \(\alpha_i = 1\), hence \(\fv_i = \fe_i\). Therefore, node $i$ is a root.
\end{proof}

The subspace-membership condition \(\fe_i \otimes \fb_i \in A\) is equivalent to \(\fe_i \otimes \fe_i \otimes \fb_i\) being a rank-one summand of $T$.
This says the coefficient of \(\fe_i \otimes \fe_i\) in each precision matrix $\Theta_j$ is the same across all contexts \(j\) in which $i$ is not intervened, cf. Proposition~\ref{prop:rank-reduction}.
Thus, \(\fe_i \otimes \fb_i \in A\) implies that the variance
of node \(i\) is stable across the contexts in which it is not intervened.

We compare our approach to the simpler idea of checking the consistency of the variance of a variable across contexts. Doing so, one recovers the roots of the DAG together with variables that look like roots based on the pattern of interventions, as follows.

\begin{lemma}\label{lem: nodes that have stable second statistics}
Under Assumption~\ref{ass:lsem}, fix \(i \in [p]\). Consider an intervention context in which neither $i$ nor any of its ancestors are intervened on. Then the variance of \(x_i\) is the same as in the observational context. 
\end{lemma}

\begin{proof}
If none of the ancestors of \(i\) are intervened on, then the structural equations for \(i\) and all of its ancestors are unchanged, and so the joint distribution of \(i \) and its ancestors is unchanged. Since \(x_i\) is a function only of its ancestors and its own noise variable, and since \(i\) is not intervened on, the distribution of \(x_i\) is unchanged. Hence its variance is unchanged.
\end{proof}

We compare Lemmas~\ref{lem: proj_norm} and~\ref{lem: nodes that have stable second statistics}. 
Stable variance is only a necessary condition for a node to be a root, since non-root nodes whose ancestors are never intervened also satisfy this property. 
In comparison, the condition \(\fe_i \otimes \fb_i \in A\) is necessary and sufficient for \(i\) to be a root, because it also incorporates information from contexts where \(i\) is intervened.

\section{Algorithm}\label{sec:algorithm}

We recover a causal order of the nodes in LSEMs with perfect interventions. This is achieved in two steps: a root candidate selection procedure (Section~\ref{sec: proj_norm}) followed by a pairwise refinement step (Section~\ref{sec: pair wise order}). The two-step algorithm design is validated by the ablation experiment in Appendix~\ref{app:ablation}. Estimation of the matrix \(\Lambda\) given the ordering is a standard problem and is deferred to the appendix. 
First we explain how to construct the sample precision tensor.

\subsection{The sample precision tensor}

So far, we have assumed access to the true population covariance matrices, and used these to construct the precision tensor. 
In practice, we recover a causal order from finite samples. 
Let \(X_1\in \RR^{n_1\times p} \) be the observational dataset and \( X_2\in \RR^{n_2\times p}, \ldots,X_k\in \RR^{n_k\times p}\) be the interventional datasets. 
We construct the sample precision tensor in Algorithm~\ref{alg:sample_precision}, see Figure~\ref{fig:precision tensor construction} for an illustration. 

\begin{algorithm}[htbp]
\caption{Sample precision tensor}
\label{alg:sample_precision}
\SetAlgoLined

\KwIn{Sample covariance matrices for contexts $1, \ldots, k$; intervention-pattern matrix \( B \in \RR^{p \times k}\).}

\KwOut{The sample precision tensor.}

\For{\(j=1,\ldots,k\)}{
    \(M_j \gets \Cov(X_j)^{-1}\)\;
    
    \For{\(i=1,\ldots,p\)}{
        \If{\(B_{ij}=0\)}{
            \(M_j \gets M_j-\dfrac{1}{\Cov(X_j)_{i,i}}\,\fe_i\fe_i\T\)\;
        }
    }
}

Form the tensor \(T=[M_1|\cdots|M_k]\in\RR^{p\times p\times k}\)\;

\Return{\( T\)}\;

\end{algorithm}

The sample precision tensor only requires access to the sample covariance matrices $\Cov(X_1),\ldots,\Cov(X_k) \in \RR^{p \times p}$, not to individual data samples. Later for root selection we need the number of samples in each context, but again not the individual samples.

\subsection{Greedy root candidate selection}\label{sec: proj_norm}

We use a greedy procedure to identify root nodes from the precision tensor.
Recall that root nodes are those satisfying \(\fe_i \otimes \fb_i \in A\), see Lemma \ref{lem: proj_norm}. In practice, to quantify how close a node is to satisfying this condition, we compute a subspace-proximity score from projection onto the subspace $A$:
\begin{equation}
    \label{eqn:proj_score}
    \alpha_i = \frac{\|P_A(\fe_i \otimes \fb_i)\|_F}{\|\fb_i\|} \in [0,1],
\end{equation}
where \(P_A\) denotes orthogonal projection onto \(A\).

\begin{lemma}
Node \(i\) is a root if and only if \(\alpha_i=1\).
\end{lemma}

\begin{proof}
Node \(i\) is a root if and only if
\(
\fe_i\otimes \fb_i \in A,
\)
by Lemma~\ref{lem: proj_norm}.
This is equivalent to
\[
\|P_A(\fe_i\otimes \fb_i)\|_F
=
\|\fe_i\otimes \fb_i\|_F.
\]
Since \(\|\fe_i\|=1\), we have
\(
\|\fe_i\otimes \fb_i\|_F=\|\fb_i\|,
\)
so this is equivalent to
\(
\frac{\|P_A(\fe_i\otimes \fb_i)\|_F}{\|\fb_i\|}
=1
\).
\end{proof}

In the population setting, the top scores equal 1 and exactly identify the root nodes.
The following result shows that if a node's projection score given by~\eqref{eqn:proj_score} is high, then the node is close to being a root, in the sense that vectors $\fe_i$ and $\fv_i$ have similar directions. 

\begin{lemma}\label{lem: high projection score}
Let
\(
A=\mathrm{span}\{\fv_j\otimes \fb_j:j=1,\ldots,p\}
\),
where $\fv_j, \fb_j$ are as in Proposition~\ref{prop:rank-one-sums-observational}.
Let \(\fw_1,\ldots,\fw_p\) satisfy
\(
\langle \fw_\ell,\fv_j\rangle=\delta_{\ell j}.
\)
Let $\alpha_i$ be the projection score from~\eqref{eqn:proj_score} and define 
$\delta_i=\sqrt{1-\alpha_i^2}$.
Assume the $\fb_i$ are separated, in that for every \(j\neq i\),
\begin{equation}\label{eq: separation}
\inf_c\|\fb_i-c\fb_j\|
\geq \gamma_i \|\fb_i\| \quad \text{ for some} \quad \gamma_i>0.
\end{equation}
Then
\[
\left\|
\fe_i-\langle \fw_i,\fe_i\rangle \fv_i
\right\|
\leq
\left(\sum_{j\neq i}
\frac{\|\fw_j\|\|\fv_j\|}{\gamma_i}
\right)\delta_i.
\]
\end{lemma}
\begin{proof}
When \(\alpha_i\) is close to \(1\), the matrix \(\fe_i\otimes \fb_i\) is well approximated by a linear combination of the matrices \(\fv_j\otimes \fb_j\).
There exist scalars \(c_1,\ldots,c_p\) such that
\begin{equation}\label{eq: error matrix}
\left\|
\fe_i\otimes \fb_i
-
\sum_{j=1}^p c_j\fv_j\otimes \fb_j
\right\|_F
\leq
\|\fb_i\|\delta_i,
\end{equation}
since 
\[
\dist(\fe_i\otimes \fb_i,A)
=
\|\fe_i\otimes \fb_i\|\sqrt{1-\alpha_i^2}
=
\|\fb_i\|\delta_i.
\]
Multiplying the matrix in \eqref{eq: error matrix} by $\fw_\ell$, we obtain
\begin{align}
\left\|
\langle \fw_\ell,\fe_i\rangle \fb_i
-
c_\ell \fb_\ell
\right\|
& = \| \fw_\ell\T (\fe_i\otimes \fb_i
-
\sum_{j=1}^p c_j\fv_j\otimes \fb_j)\| \\
&\leq \|\fw_\ell\|\left\|
\fe_i\otimes \fb_i
-
\sum_{j=1}^p c_j\fv_j\otimes \fb_j
\right\|_2  \\
& \leq \|\fw_\ell\|\left\|
\fe_i\otimes \fb_i
-
\sum_{j=1}^p c_j\fv_j\otimes \fb_j
\right\|_F 
\leq
\|\fw_\ell\|\|\fb_i\|\delta_i, \label{eq: bound w_l e_i}
\end{align}
where the equality follows from
\( \langle \fw_\ell,\fv_j\rangle=\delta_{\ell j} \).
By the separation assumption \eqref{eq: separation},
\begin{equation}\label{eq: use separation}
\inf_c
\|\langle \fw_\ell,\fe_i\rangle \fb_i-c\fb_\ell\|
=
|\langle \fw_\ell,\fe_i\rangle|
\inf_c\|\fb_i-c\fb_\ell\|
\geq
|\langle \fw_\ell,\fe_i\rangle|\gamma_i\|\fb_i\|.
\end{equation}
Combining \eqref{eq: bound w_l e_i} and \eqref{eq: use separation} gives $|\langle \fw_\ell,\fe_i\rangle|\gamma_i\|\fb_i\|
\leq
\|\fw_\ell\| \|\fb_i\| \delta_i
$ which implies 
\[
|\langle \fw_\ell,\fe_i\rangle|
\leq
\frac{\|\fw_\ell\|}{\gamma_i}\delta_i.
\]
Finally, using the dual-basis expansion
\(
\fe_i=\sum_{\ell=1}^p \langle \fw_\ell,\fe_i\rangle \fv_\ell,
\)
we get
\[
\fe_i-\langle \fw_i,\fe_i\rangle \fv_i
=
\sum_{\ell\neq i}\langle \fw_\ell,\fe_i\rangle \fv_\ell.
\]
Taking norms and applying the bound above gives
\[
\left\|
\fe_i-\langle \fw_i,\fe_i\rangle \fv_i
\right\|
\leq
\sum_{\ell\neq i}
|\langle \fw_\ell,\fe_i\rangle|\,\|\fv_\ell\|
\leq
\left(\sum_{\ell\neq i}
\frac{\|\fw_\ell\|\|\fv_\ell\|}{\gamma_i}
\right)\delta_i.\qedhere
\]
\end{proof}

Searching over the space of rank one matrices, or tensors, to find those that lie close to a subspace is behind tensor decomposition algorithms including~\cite{kileel2025subspace,wang2025multi}. In TSCD, this idea is adapted to testing the membership of specific rank one matrices.

\begin{algorithm}[htbp]
\caption{Greedy causal order recovery}
\label{alg:main_order}
\SetAlgoLined

\KwIn{Sample precision tensor $T \in \RR^{p \times p \times k}$; intervention-pattern matrix \( B \in \RR^{p \times k}\); number of candidates \(n\).}

\KwOut{A causal order \(\pi=(\pi_1,\ldots,\pi_p)\).}

Initialize \( T \), \(S\gets \{1,\ldots,p\}\), and \(\pi\gets (\,)\)\;

\For{\(t=1,\ldots,p\)}{
    Compute the projection operator \(P_A\) onto the subspace spanned by the slices $T(1,:,:),\ldots,T(p,:,:)$ of the current tensor \(T\)\;
    
    \ForEach{\(i\in S\)}{
        \[
        \alpha_i \gets \frac{\|P_A(\fe_i\otimes \fb_i)\|_F}{\|\fb_i\|}.
        \]
    }
    
    Let \(C_t\subseteq S\) be the set of nodes with the top \(n\) values of \(\alpha_i\)\;
    
    \(\pi_t \gets \textsc{RootSelection}(C_t,\Cov(X_j)\}_{j=1}^k,B)\)\tcp*{Algorithm~\ref{alg:pairwise_asym_corr}}
    
    Append \(\pi_t\) to \(\pi\)\;
    
    Remove \(\pi_t\) from \(S\)\;
    
    Delete the \(\pi_t\)-th row and column from each slice $T(:,:,1),\ldots, T(:,:,k)$ of \(T\)\;
}

\Return{\(\pi\)}\;

\end{algorithm}

\subsection{Pairwise intervention-asymmetric correlation tests}\label{sec: pair wise order}

In practice, roots can have scores less than one and non-root nodes may have high scores. See Appendix~\ref{app:root projection scores} for empirical behavior of the scores at roots.
To improve robustness, we retain a small set of top candidates with high $\alpha_i$ scores.

We then use a refinement step that resolves ambiguity among candidate root nodes using asymmetry in correlations across contexts.
Consider two nodes \(i\) and \(j\) that are correlated in the observational context. If intervening on \(i\) leaves the correlation with \(j\) nonzero, then the dependence cannot be explained by paths ending at \(i\); it must flow from \(i\) to \(j\), so \(i\) is an ancestor of \(j\). If instead the correlation disappears, the dependence must arise from a path into \(i\) (either \(j \to i\) or a latent confounder), and thus \(i\) is not a root.
This is formalized in the following lemma.

\begin{lemma}
Let \(i\) and \(j\) be two nodes in a DAG. Assume \(\mathrm{corr}(x_i, x_j)\neq 0\) in the observational context. If there is a context where \(i\) is intervened on and \(j\) is not and in which \(\mathrm{corr}(x_i, x_j)\neq 0\), then \(i\) is an ancestor of \(j\). Otherwise, either \(j\) is an ancestor of \(i\) or \(i\) and \(j\) share an ancestor.
\end{lemma}

\begin{figure}[h]
\centering
\begin{tikzpicture}[
    node distance=2.2cm,
    var/.style={circle, draw, minimum size=0.8cm},
    int/.style={circle, draw, fill=gray!20, minimum size=0.8cm},
    >=stealth
]

\draw[dashed, gray, thick] (3.5,-1.2) -- (3.5,1.3);
\draw[dashed, gray, thick] (10.5,-1.2) -- (10.5,1.3);

\node[int] (i1) at (0,0) {$i$};
\node[var] (j1) at (2.5,0) {$j$};
\draw[->, thick] (i1) -- (j1);
\node at (1.25,-0.7) {\small dependence remains};
\node at (1.25,0.8) {\small \(i\) intervened, \(j\) not};

\node[int] (i2) at (5.75,0) {$i$};
\node[var] (j2) at (8.25,0) {$j$};
\draw[->, thick] (j2) -- (i2);
\node at (7,-0.7) {\small dependence broken};
\node at (7,0.8) {\small intervening \(i\) removes incoming edge};

\node[int] (i3) at (11.5+1,0) {$i$};
\node[var] (j3) at (14+1,0) {$j$};
\node[var] (u3) at (12.75+1,0.25) {$l$};
\draw[->, thick] (u3) -- (i3);
\draw[->, thick] (u3) -- (j3);
\node at (12.75+1,-0.7) {\small dependence broken};
\node at (12.75+1,0.8) {\small confounding path into \(i\) removed};

\end{tikzpicture}
\caption{Intervention-asymmetric correlation test. If intervening on \(i\) preserves correlation with \(j\), then the dependence must flow from \(i\) to \(j\), so \(i\) is an ancestor of \(j\). If the correlation disappears, the observational dependence may instead come from \(j\to i\) or from a latent common cause.
}
\label{fig:pairwise-asym-corr}
\end{figure}
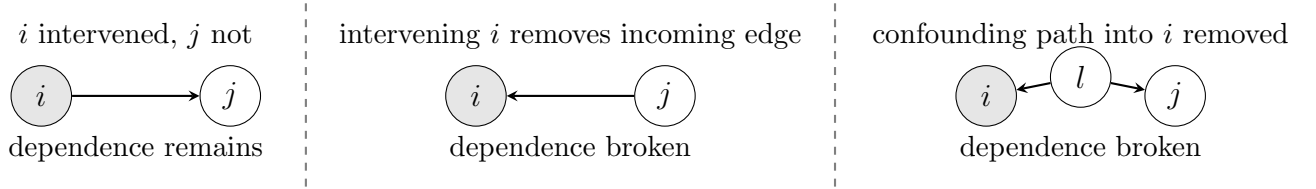

\subsection{Root selection algorithm}

We use pairwise intervention-asymmetric tests to select a root from a candidate set~\(C\). For each pair \(i,j\in C\), we compare how their correlation behaves across intervention contexts. Correlations are classified as \texttt{NONZERO}, \texttt{ZERO}, or \texttt{INCONCLUSIVE} using hypothesis tests (Appendix~\ref{app:hypothesis}).
If the correlation between \(i\) and \(j\) is \texttt{NONZERO} when intervening on \(i\) but not when intervening on \(j\), this is evidence that \(i\) is upstream of \(j\). Conversely, \texttt{ZERO} correlation weakens this claim. 
We aggregate these pairwise comparisons into node scores and select the node with the highest score as the next root in the causal order.
See Algorithm~\ref{alg:pairwise_asym_corr} for details. 

\begin{algorithm}[htbp]
\caption{\textsc{RootSelection} by pairwise correlation tests}
\label{alg:pairwise_asym_corr}
\SetAlgoLined

\KwIn{Candidate set $C$; sample covariance matrices for contexts $\ell = 1, \ldots, k$; intervention-pattern matrix $B$; sample sizes $\{N_\ell\}$.}

\KwOut{Selected root $r\in C$.}

Initialize $\mathrm{score}(i)\gets 0$ for all $i\in C$\;

\ForEach{pair $i\neq j$ in $C$}{
    \For{$\ell=1,\ldots,k$}{
        Compute the sample correlation $\rho_{ij}^{(\ell)}$ in each context.\;
        
        Classify $\rho_{ij}^{(\ell)}$ as \texttt{NONZERO}, \texttt{ZERO}, or \texttt{INCONCLUSIVE} for all $\ell$.\;
    }

    Let $\mathcal{C}_{\text{obs}}, \mathcal{C}_{i}, \mathcal{C}_{j}$ be the collections of contexts in which neither, only $i$, or only $j$ are intervened, respectively\;
    
Initialize $s_i\gets 0$ and $s_j\gets 0$\;

\If{$\mathcal{C}_{\text{obs}}$ is \texttt{NONZERO}}{
    \If{\texttt{NONZERO} appears in $\mathcal{C}_i$ }{
        $(s_i,s_j)\gets (s_i+1, s_j-1)$\;
    }
    \Else{
        $s_i\gets s_i-1$\;
    }

    \If{\texttt{NONZERO} appears in $\mathcal{C}_{j}$ }{
       $(s_i,s_j)\gets (s_i-1, s_j+1)$\;
    }
    \Else{
        $s_j\gets s_j-1$\;
    }
}

\If{$s_i>s_j$}{
    $\mathrm{score}(i)\gets \mathrm{score}(i)+1$\;
    $\mathrm{score}(j)\gets \mathrm{score}(j)-1$\;
}
\ElseIf{$s_j>s_i$}{
    $\mathrm{score}(j)\gets \mathrm{score}(j)+1$\;
    $\mathrm{score}(i)\gets \mathrm{score}(i)-1$\;
}
}

$r\gets \arg\max_{i\in C}\mathrm{score}(i)$\;

\Return{$r$}\;

\end{algorithm}

One may wonder whether pairwise correlation tests alone are sufficient for causal order. 
While such tests exploit intervention asymmetry, they rely only on local, pairwise information and can be unreliable in finite samples or when intervention patterns are limited, e.g., when there are few or no contexts in which one node is intervened on while the other is not. 
The importance of combining both global projection norm-based candidate selection and local pairwise correlation testing steps is shown in Appendix~\ref{app:ablation}.

\section{Experiments}\label{sec:experiments}

We evaluate TSCD on synthetic linear and nonlinear structural equation models, comparing against causal discovery baselines. Experiment details are in Appendix~\ref{app:exp_details}. 
Experiments were conducted on a MacBook Pro with an Apple M2 chip and on a university computing cluster. 
See \href{https://github.com/QWE123665/Tensor-based-Second-order-Causal-Discovery}{this repository} for the code. 

\subsection{Linear SEM}

We generate data from random LSEMs with \(p=10\) nodes and edge probability \(0.6\), under one observational and four intervention contexts. We consider three noise regimes (Gaussian, heavy-tailed, and mixed).
We compare three classes of methods:

\begin{enumerate}
    \item \textbf{Interventional methods} (ours, GIES, IGSP), which use interventional data and intervention patterns;

    \item \textbf{Observational methods with adjacency matrix output} 
    (LiNGAM, SortRegress, NOTEARS), applied separately to each context and combined;

    \item \textbf{Observational methods with graph output} (GES, PC), applied to five times as many observational samples to match the total sample size used by the other methods.
\end{enumerate}
 
Performance is evaluated using the relative Frobenius error,
\[
\frac{\|{\Lambda_{\text{est}}}-\Lambda\|_F}{\|\Lambda\|_F},
\]
for methods estimating \(\Lambda\), the F1 score for edge recovery,
\[
\mathrm{F1}
=
\frac{2\,\mathrm{precision}\cdot\mathrm{recall}}
{\mathrm{precision}+\mathrm{recall}},
\]
and runtime.
Results are shown in Figure~\ref{fig:lsem}. Our method achieves high accuracy and computational efficiency, performing comparably to LiNGAM in non-Gaussian settings and outperforming baselines when there is Gaussian noise.

\begin{figure}[htbp]
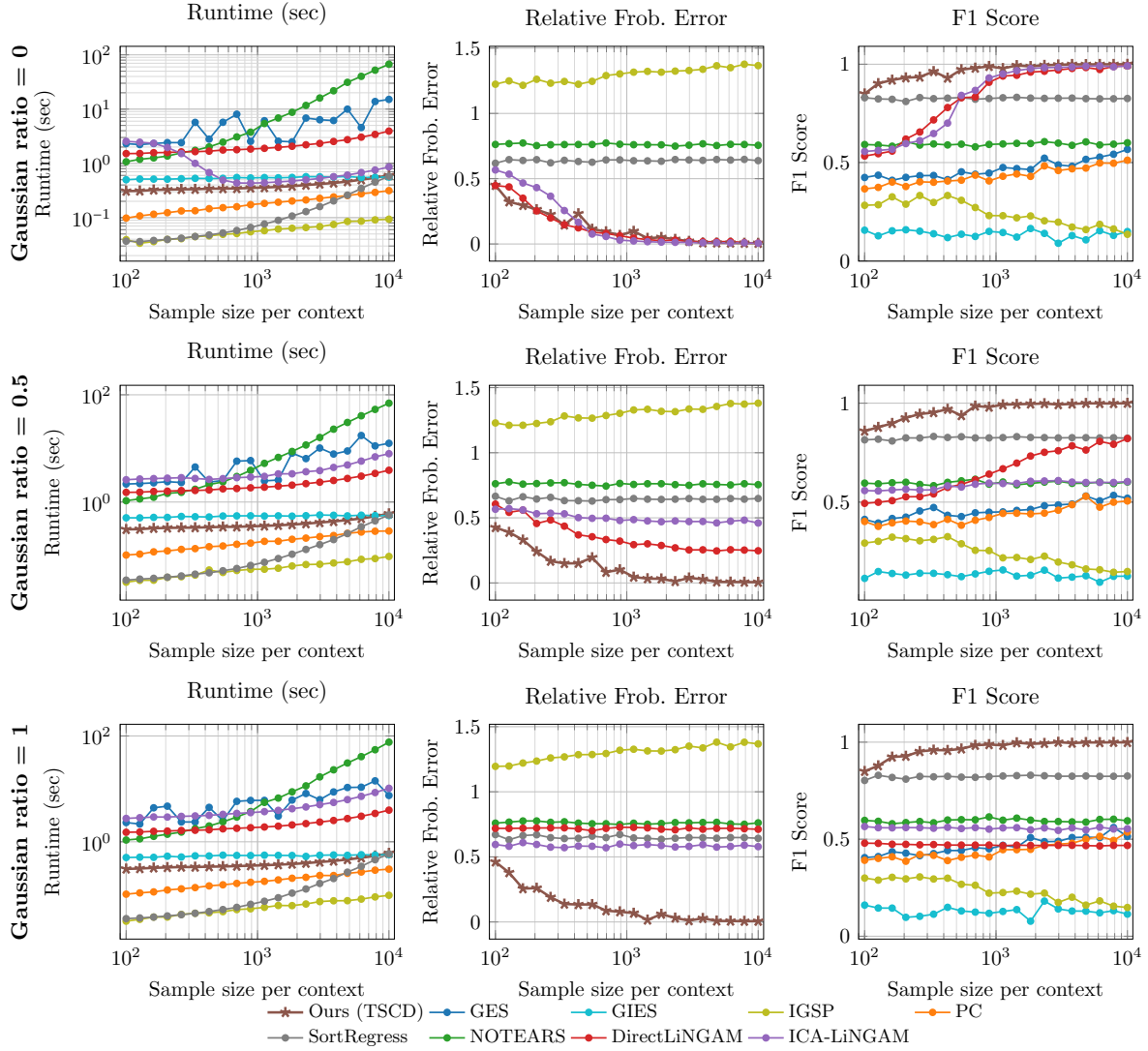

  \centering
  \resizebox{0.95\linewidth}{!}{\input{metrics_gaussian_ratio_0.tex}}
  \resizebox{0.95\linewidth}{!}{\input{metrics_gaussian_ratio_0p5}}
  \resizebox{0.95\linewidth}{!}{\input{metrics_gaussian_ratio_1}}
  \caption{Performance comparison across different noise settings, with Gaussian ratios 0, 0.5, and 1.}\label{fig:lsem}
\end{figure}

\subsection{Scalability}
\label{sec:scalability}

We consider graphs with up to \(p=400\) nodes under sparse LSEMs with perfect interventions.
Figure~\ref{fig:scalability} reports F1 score and runtime as functions of \(p\). The F1 score remains stable above  \(0.94\) across the range, with no degradation as the problem size increases. Runtime grows from \(0.13\)s at \(p=50\) to \(214\)s at \(p=400\). The experiment was conducted on a MacBook Pro with an Apple M2 chip. 

\begin{figure}[ht]
  \centering
  \resizebox{0.7\linewidth}{!}{\begin{tikzpicture}
  \begin{axis}[
      name=f1plot,
      xlabel={$p$},
      ylabel={F1 score},
      title={F1 score vs. number of nodes},
      ymin=0.925, ymax=1.025,
      grid=both,
      grid style={gray!20},
      mark options={scale=0.7},
    ]
    \addplot[
      blue,
      mark=o,
      error bars/.cd,
      y dir=both,
      y explicit,
    ] coordinates {
      (50,  0.9567202745277145) +- (0, 0.0223165979320993)
      (60,  0.9720704222271512) +- (0, 0.0192063809462819)
      (70,  0.9695894935084987) +- (0, 0.0269885121058773)
      (80,  0.9659372707214168) +- (0, 0.0351749145371431)
      (90,  0.9734288288630170) +- (0, 0.0240204937041928)
      (100, 0.9864513050995323) +- (0, 0.0111868963738243)
      (110, 0.9806734263966396) +- (0, 0.0175031432731765)
      (120, 0.9794577128328982) +- (0, 0.0174999516171446)
      (130, 0.9810767865160308) +- (0, 0.0227912305945089)
      (140, 0.9899590464365016) +- (0, 0.0094290863887802)
      (150, 0.9922226647639393) +- (0, 0.0063042738072925)
      (160, 0.9883494122375390) +- (0, 0.0079248974590588)
      (170, 0.9918263992519301) +- (0, 0.0073002521766836)
      (180, 0.9949108938197385) +- (0, 0.0053492355728223)
      (190, 0.9920197265161044) +- (0, 0.0115126001611834)
      (200, 0.9962790197031419) +- (0, 0.0021934603756517)
      (210, 0.9961371203361254) +- (0, 0.0025778926405671)
      (220, 0.9966410025761954) +- (0, 0.0057640554963970)
      (230, 0.9958374454723365) +- (0, 0.0021500808176233)
      (240, 0.9945411222769758) +- (0, 0.0048408615009742)
      (250, 0.9963468100465560) +- (0, 0.0025262240942785)
      (260, 0.9957802362937536) +- (0, 0.0034442892800583)
      (270, 0.9970598757492770) +- (0, 0.0030653319084440)
      (280, 0.9970535179845225) +- (0, 0.0016575203218928)
      (290, 0.9976172062284695) +- (0, 0.0017350391500544)
      (300, 0.9983825685394618) +- (0, 0.0010213312296107)
      (310, 0.9981588903484990) +- (0, 0.0017997897974893)
      (320, 0.9969057540112282) +- (0, 0.0025867394649654)
      (330, 0.9979884765347855) +- (0, 0.0016232299947520)
      (340, 0.9986625439538288) +- (0, 0.0012311092089989)
      (350, 0.9970892506972566) +- (0, 0.0030194123394326)
      (360, 0.9983159748514460) +- (0, 0.0014783537532008)
      (370, 0.9977039496164716) +- (0, 0.0019201221568590)
      (380, 0.9971504161527843) +- (0, 0.0031452223411092)
      (390, 0.9984452983807742) +- (0, 0.0016909550288724)
      (400, 0.9963198375895164) +- (0, 0.0033542220447340)
    };
  \end{axis}

  \begin{axis}[
      at={(f1plot.outer east)}, anchor=outer west,
      xlabel={$p$},
      ylabel={runtime (s)},
      title={Runtime vs. number of nodes},
      grid=both,
      grid style={gray!20},
      mark options={scale=0.7},
    ]
    \addplot[
      green!60!black,
      mark=o,
      error bars/.cd,
      y dir=both,
      y explicit,
    ] coordinates {
      (50,  0.1612808123998548) +- (0, 0.0464098338961806)
      (60,  0.2606885248997060) +- (0, 0.0845397040349237)
      (70,  0.7567445207998389) +- (0, 0.8281307063137766)
      (80,  0.7985257460999492) +- (0, 0.3461818489422592)
      (90,  0.8317659209997146) +- (0, 0.1178150953190129)
      (100, 1.1151397332998385) +- (0, 0.2544686521967156)
      (110, 1.2655249875999288) +- (0, 0.3245177607099380)
      (120, 2.1390525292999882) +- (0, 1.4769622281744019)
      (130, 2.6985196041005111) +- (0, 1.1093961952077103)
      (140, 3.9236471291002091) +- (0, 0.7904486906797931)
      (150, 3.8910020959001486) +- (0, 0.5548385247292711)
      (160, 5.4019059418002140) +- (0, 1.2142505608227896)
      (170, 5.2337228831996985) +- (0, 0.6553527180930825)
      (180, 5.9388335040001037) +- (0, 0.6905412198344302)
      (190, 6.8225534999999580) +- (0, 0.7096866018534449)
      (200, 8.4981171208999520) +- (0, 1.0769401063473065)
      (210, 9.2012919831999174) +- (0, 1.2247190836738053)
      (220, 9.6433036959000678) +- (0, 0.5140383908493146)
      (230, 11.8208248834003822) +- (0, 0.6060198978397972)
      (240, 12.5711084876000321) +- (0, 0.7063664400520450)
      (250, 17.0688888042004692) +- (0, 1.7255704416861137)
      (260, 20.0350682919000960) +- (0, 1.6940706391650671)
      (270, 24.5911729250001372) +- (0, 3.6085548523366926)
      (280, 25.6652304085004275) +- (0, 0.5763007982856730)
      (290, 32.3086848128001307) +- (0, 4.8469238499647433)
      (300, 35.7681752708998815) +- (0, 1.6305994581906216)
      (310, 35.5072569417006889) +- (0, 1.2152344363888481)
      (320, 37.5280111959000351) +- (0, 4.1611004449683833)
      (330, 40.8205650583997368) +- (0, 1.8397156367102852)
      (340, 45.1928820415996597) +- (0, 2.5350773008561314)
      (350, 50.3441207667998256) +- (0, 2.9642565278544573)
      (360, 62.8867249667004202) +- (0, 5.1601377768178089)
      (370, 65.8595781164993213) +- (0, 6.6723192818416104)
      (380, 86.2968754873003547) +- (0, 22.3341294423934507)
      (390, 107.4018354168001679) +- (0, 8.6356662176702521)
      (400, 128.6284492416008902) +- (0, 10.4470399204899298)
    };
  \end{axis}
\end{tikzpicture}}
  \caption{Scalability of TSCD.
  Left: F1 score versus number of nodes.
  Right: runtime versus number of nodes.
  Each point is averaged across 10 random sparse DAGs, each with sample size \(p^2\). The error bars are given by the standard deviation across trials. }
  \label{fig:scalability}
\end{figure}
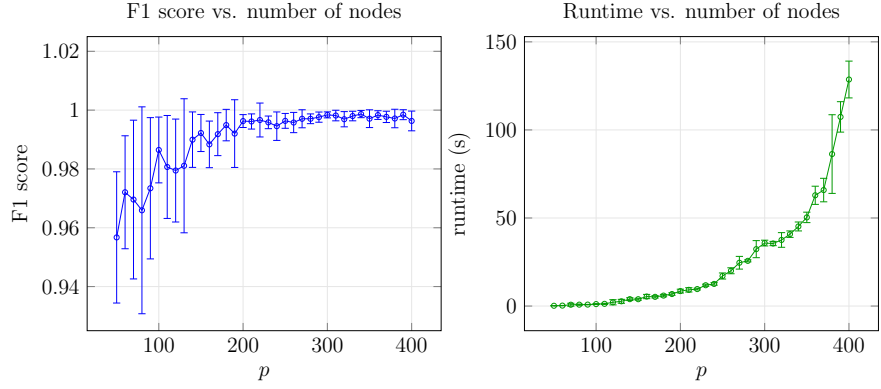

\subsection{TSCD-nonlinear}

We consider nonlinear SEMs with perfect interventions. The goal is to test whether we can recover the causal order and the DAG based on stability of each node when the structural functions are nonlinear. The study of nonlinear SEMs is widespread, also appearing in e.g.~\cite{zheng2020learning,monti2020causal,hoyer2008nonlinear}. 

TSCD-nonlinear constructs the causal order greedily based on stability across contexts. The initial root is selected using marginal variance stability. At each step, we predict each candidate node from previously selected variables using a neural network trained on pooled non-intervened contexts, and select the node whose prediction error is most stable across contexts. Parent sets are estimated via a gated network for feature selection.

We compare TSCD-nonlinear against a baseline that orders nodes by marginal variance. 
The experiment is conducted on a random  DAG with \(p=10\) nodes and edge probability \(0.8\).
Ordering accuracy is measured by the number of parent-child errors, edges where the parent appears after the child in the recovered order.

Our method recovers the correct causal order, while the baseline produces 13 parent-child errors. A representative example is in Figure~\ref{fig:DAGs nonlinear}. Our recovered graph has an edge error rate of \(0.34\). The runtime is 134 seconds using 5 parallel workers.

\begin{figure}[htbp]
    \centering
    \includegraphics[width=0.38\linewidth]{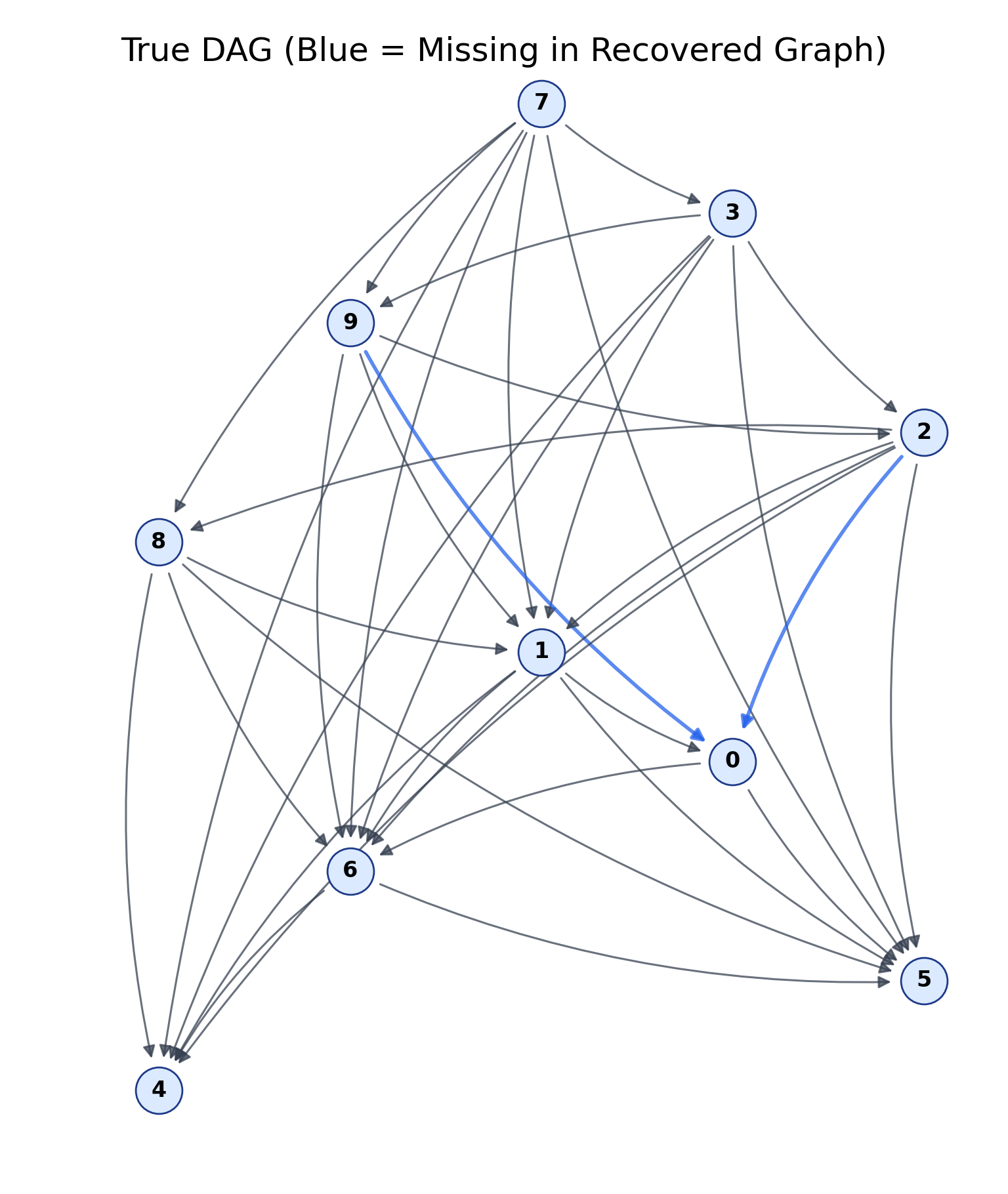}
    \includegraphics[width=0.38\linewidth]{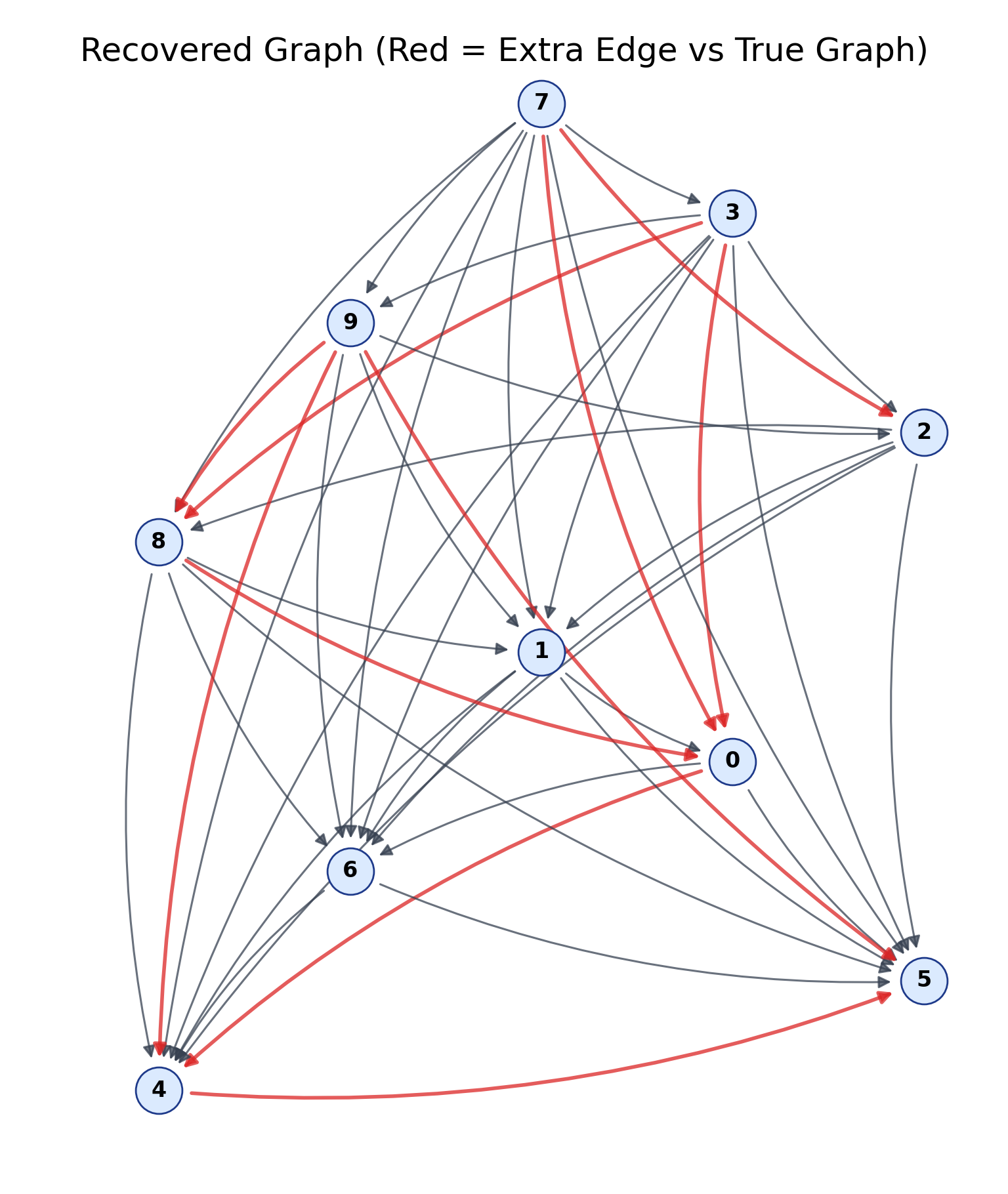}
    \caption{Ground-truth DAG (left) and recovered graph (right) for a nonlinear SEM. The recovered structure matches the true causal order, with an edge error rate of 0.34.}
    \label{fig:DAGs nonlinear}
\end{figure}

\subsection{Light Causal Chamber}
We evaluate TSCD and TSCD-nonlinear on a real-world light-chamber experiment from the Causal Chamber \cite{gamella2025causal}. 
The chamber is a controlled physical system in which light sources and sensor readings interact through the device. We use the light-tunnel intervention dataset \href{https://github.com/juangamella/causal-chamber/tree/main/datasets/lt_interventions_standard_v1}{\texttt{lt\_interventions\_standard\_v1}}. To focus on a setting where the intervention targets can be treated as perfect interventions, we restrict to a partial ground-truth graph over nine variables:
\[
R, G, B, \tilde I_1, \tilde V_1, L_{11}, L_{12}, D^V_1, T^V_1.
\]
Here \(R,G,B\) denote the brightness of the red, green, and blue LEDs respectively on the main light source. The variables \(\tilde I_1\) and \(\tilde V_1\) are the uncalibrated infrared and visible-light intensity measurements from the first light sensor, which is placed before the polarizers. The variables \(L_{11}\) and \(L_{12}\) are the brightness settings of the two auxiliary LEDs placed by this first light sensor. Finally, \(D^V_1\) and \(T^V_1\) are visible-light sensor parameters for the first sensor: the selected visible photodiode and the photodiode exposure time.

We compare against several standard causal discovery baselines. TSCD-nonlinear achieves high accuracy, recovering all ground-truth edges in the selected graph with precision \(0.80\) and recall \(1.00\). TSCD recovers the linear part of the graph, the edges from \(R,G,B\) to the two sink variables \(\tilde I_1\) and \(\tilde V_1\), yielding precision \(0.857\) and recall \(0.500\). In contrast, the competing methods have lower recall or precision on this benchmark, see Table \ref{tab:causal-chamber}.

\begin{table}[t]
\centering
\begin{tabular}{lcc}
\toprule
Method & Precision & Recall \\
\midrule
TSCD & 0.857 & 0.500 \\
TSCD-nonlinear & 0.800 & 1.000 \\
SortRegress & 0.038 & 0.500 \\
LiNGAM (ICA) & 0.028 & 0.500 \\
LiNGAM (Direct) & 0.019 & 0.500 \\
GIES & 0.833 & 0.417 \\
IGSP & 0.000 & 0.167 \\
\bottomrule
\end{tabular}
\caption{Causal discovery performance on the selected Causal Chamber light-tunnel graph.}
\label{tab:causal-chamber}
\end{table}

\subsection{Flow-cytometry Dataset}

We consider the flow-cytometry dataset~\cite{sachs2005causal}, which records single-cell protein phosphorylation for $p = 11$ proteins across nine experimental conditions. 
We pool the two unperturbed baselines into one observational environment, and use the five kinase-inhibitor conditions as intervention contexts.
We exclude the two activator conditions (PMA and \(\beta\)2-cAMP), since inhibitor interventions are closer to perfect interventions.
We compare the graph recovered by our method in Table~\ref{tab:sachs-edges} with two reference DAGs from \cite{sachs2005causal}  and graphs recovered by ICP and hiddenICP~\cite{peters2016causal,meinshausen2016methods}. 

TSCD recovers eight edges, seven of which appear in the DAG from~\cite{sachs2005causal} or recovered by hiddenICP.
Most missing edges are either among \(\{\mathrm{Raf}, \mathrm{PLCg}, \mathrm{Erk}, \mathrm{PKA}, \mathrm{p38}, \mathrm{JNK}\}\), which are not intervened in any context so are not identifiable, or involve $\mathrm{PKA}, \mathrm{PKC}$, where feedback loops and nonlinear effects are expected~\cite{itani2010structure}.

\begin{table}[htbp]
\centering
\small
\begin{tabular}{lccccc}
\toprule
Edge & \cite{sachs2005causal}a & \cite{sachs2005causal}b & \cite{peters2016causal}ICP & \cite{peters2016causal}hiddenICP & TSCD (ours) \\
\midrule
Raf $\to$ Mek          & $\checkmark$ & $\checkmark$ &            & $\checkmark$ &                                  \\
Mek $\to$ Raf          &            &            &            & $\checkmark$ & $\checkmark$                      \\
Mek $\to$ Erk          & $\checkmark$ & $\checkmark$ &            &            &                                  \\
PLCg $\to$ PIP2        & $\checkmark$ & $\checkmark$ & $\checkmark$ & $\checkmark$ & $\checkmark$                       \\
PLCg $\to$ PIP3        &            & $\checkmark$ &            &            & $\checkmark$                    \\
PLCg $\to$ PKC         & $\checkmark$ &            &            &            &                                  \\
PIP2 $\to$ PLCg        &            &            & $\checkmark$ &            &                                  \\
PIP2 $\to$ PKC         & $\checkmark$ &            &            &            &                                  \\
PIP3 $\to$ PLCg        & $\checkmark$ &            &            &            &                                  \\
PIP3 $\to$ PIP2        & $\checkmark$ & $\checkmark$ & $\checkmark$ & $\checkmark$ & $\checkmark$                         \\
PIP3 $\to$ Akt         & $\checkmark$ &            &            &            &                                  \\
Akt $\to$ Erk          &            &            & $\checkmark$ & $\checkmark$ & $\checkmark$                       \\
Akt $\to$ PKA          &            &            &            &            & $\checkmark$                     \\
Erk $\to$ Akt          &            & $\checkmark$ & $\checkmark$ & $\checkmark$ &                                  \\
\textcolor{red}{PKA $\to$ Raf}  & $\checkmark$ & $\checkmark$ &            &            &                     \\
PKA $\to$ Mek          & $\checkmark$ & $\checkmark$ &            &            &                                  \\
\textcolor{red}{PKA $\to$ Erk}  & $\checkmark$ & $\checkmark$ & $\checkmark$ &                              &                              \\
PKA $\to$ Akt          & $\checkmark$ & $\checkmark$ &            & $\checkmark$ &                                  \\
\textcolor{red}{PKA $\to$ p38}  & {$\checkmark$} & {$\checkmark$} &                              &                              &                              \\
\textcolor{red}{PKA $\to$ JNK}  & {$\checkmark$} &  {$\checkmark$} &                              &                              &                              \\
PKC $\to$ Raf          & $\checkmark$ & $\checkmark$ &            &            &                                  \\
PKC $\to$ Mek          & $\checkmark$ & $\checkmark$ &            &            &                                  \\
PKC $\to$ PKA          &            & $\checkmark$ &            &            &                                  \\
PKC $\to$ p38          & $\checkmark$ & $\checkmark$ &            & $\checkmark$ &                                  \\
PKC $\to$ JNK          & $\checkmark$ & $\checkmark$ & $\checkmark$ & $\checkmark$ &                                  \\
\textcolor{red}{p38 $\to$ JNK}  &      &         &             & {$\checkmark$}  & $\checkmark$                 \\
p38 $\to$ PKC          &            &            &            & $\checkmark$ & $\checkmark$                            \\
JNK $\to$ PKC          &            &            &            & $\checkmark$ &                                  \\
\textcolor{red}{JNK $\to$ p38}  &      &       &        & {$\checkmark$}  &                              \\
\midrule
Total                  & 18 & 17 & 7 & 13 & 8 \\
\bottomrule
\end{tabular}
\caption{Edges comparison across two reference graphs, graph recovered by ICP, hiddenICP and TSCD.
TSCD recovers a sparse set of edges that is consistent with the other graphs. Red entries are non-identifiable edges, between 
$\{\text{Raf}, \text{PLCg}, \text{Erk}, \text{PKA}, \text{p38}, \text{JNK}\}$, which TSCD cannot identify.}
\label{tab:sachs-edges}
\end{table}

\section{Conclusion}
We propose a tensor-based method for causal discovery from covariance matrices, across multiple intervention contexts, to recover both a causal graph and edge weights. Our approach is identifiable under a logarithmic number of intervention contexts and performs well in linear and nonlinear settings, scaling to graphs with hundreds of nodes.

Our method relies on assumptions such as perfect interventions and uncorrelated noise, which may be violated in real-world data. 
Future directions include extending the framework to handle other types of interventions and latent confounding.
These could include noise-shift interventions~\cite{rothenhausler2015backshift}, which also admit a tensor decomposition structure, and soft interventions~\cite{markowetz2005probabilistic} which modify rather than remove edge weights.

TSCD relies only on second-order statistics through covariance and precision matrices. Incorporating first or higher-order statistics is another interesting direction for future work.

\bigskip

\begin{footnotesize} 
\noindent \textbf{Acknowledgments.}  
NO received support from the Harvard Office of Undergraduate Research and Fellowships. AS was partially supported by NSF DMS 2608217 and an Alfred P. Sloan research fellowship.
\end{footnotesize}

\bibliographystyle{alpha}
\bibliography{references}

@article{peters2014identifiability,
  title={Identifiability of Gaussian structural equation models with equal error variances},
  author={Peters, Jonas and B{\"u}hlmann, Peter},
  journal={Biometrika},
  volume={101},
  number={1},
  pages={219--228},
  year={2014},
  publisher={Oxford University Press}
}

@article{hoyer2008nonlinear,
  title={Nonlinear causal discovery with additive noise models},
  author={Hoyer, Patrik and Janzing, Dominik and Mooij, Joris M and Peters, Jonas and Sch{\"o}lkopf, Bernhard},
  journal={Advances in neural information processing systems},
  volume={21},
  year={2008}
}

@inproceedings{monti2020causal,
  title={Causal discovery with general non-linear relationships using non-linear ICA},
  author={Monti, Ricardo Pio and Zhang, Kun and Hyv{\"a}rinen, Aapo},
  booktitle={Uncertainty in artificial intelligence},
  pages={186--195},
  year={2020},
  organization={PMLR}
}

@inproceedings{zheng2020learning,
  title={Learning sparse nonparametric dags},
  author={Zheng, Xun and Dan, Chen and Aragam, Bryon and Ravikumar, Pradeep and Xing, Eric},
  booktitle={International conference on artificial intelligence and statistics},
  pages={3414--3425},
  year={2020},
  organization={Pmlr}
}

@inproceedings{eberhardt2012number,
  title={On the number of experiments sufficient and in the worst case necessary to identify all causal relations among $N$ variables},
  author={Eberhardt, Frederick and Glymour, Clark and Scheines, Richard},
  booktitle={Proceedings of the Twenty-First Conference on Uncertainty in Artificial Intelligence},
  pages={178--184},
  year={2005}
}

@article{wang2023lower,
  title={Lower bounds on the rank and symmetric rank of real tensors},
  author={Wang, Kexin and Seigal, Anna},
  journal={Journal of Symbolic Computation},
  volume={118},
  pages={69--92},
  year={2023},
  publisher={Elsevier}
}

@article{de2006link,
  title={A link between the canonical decomposition in multilinear algebra and simultaneous matrix diagonalization},
  author={De Lathauwer, Lieven},
  journal={SIAM journal on Matrix Analysis and Applications},
  volume={28},
  number={3},
  pages={642--666},
  year={2006},
  publisher={SIAM}
}

@inproceedings{johnston2023computing,
  title={Computing linear sections of varieties: quantum entanglement, tensor decompositions and beyond},
  author={Johnston, Nathaniel and Lovitz, Benjamin and Vijayaraghavan, Aravindan},
  booktitle={2023 IEEE 64th Annual Symposium on Foundations of Computer Science (FOCS)},
  pages={1316--1336},
  year={2023},
  organization={IEEE}
}

@article{wang2025multi,
  title={Multi-subspace power method for decomposing all tensors},
  author={Wang, Kexin and Pereira, Jo{\~a}o M and Kileel, Joe and Seigal, Anna},
  journal={arXiv preprint arXiv:2510.18627},
  year={2025}
}

@article{dixit2016perturb,
  title={Perturb-Seq: dissecting molecular circuits with scalable single-cell {RNA} profiling of pooled genetic screens},
  author={Dixit, Atray and Parnas, Oren and Li, Biyu and Chen, Jenny and Fulco, Charles P and Jerby-Arnon, Livnat and Marjanovic, Nemanja D and Dionne, Danielle and Burks, Tyler and Raychowdhury, Raktima and others},
  journal={cell},
  volume={167},
  number={7},
  pages={1853--1866},
  year={2016},
  publisher={Elsevier}
}

@book{spirtes2000causation,
  title={Causation, prediction, and search},
  author={Spirtes, Peter and Glymour, Clark N and Scheines, Richard},
  year={2000},
  publisher={MIT press}
}

@article{shimizu2006lingam,
  title={A Linear Non-{G}aussian Acyclic Model for Causal Discovery},
  author={Shimizu, Shohei and Hoyer, Patrik O. and Hyvärinen, Aapo and Kerminen, Antti},
  journal={Journal of Machine Learning Research},
  volume={7},
  pages={2003--2030},
  year={2006}
}

@book{peters2017elements,
  title={Elements of causal inference: foundations and learning algorithms},
  author={Peters, Jonas and Janzing, Dominik and Scholkopf, Bernhard},
  year={2017},
  publisher={MIT press}
}

@book{landsberg2011tensors,
  title={Tensors: geometry and applications: geometry and applications},
  author={Landsberg, Joseph M},
  volume={128},
  year={2011},
  publisher={American Mathematical Soc.}
}

@article{wright1934method,
  title={The method of path coefficients},
  author={Wright, Sewall},
  journal={The annals of mathematical statistics},
  volume={5},
  number={3},
  pages={161--215},
  year={1934},
  publisher={JSTOR}
}

@article{kolda2009tensor,
  title={Tensor decompositions and applications},
  author={Kolda, Tamara G and Bader, Brett W},
  journal={SIAM review},
  volume={51},
  number={3},
  pages={455--500},
  year={2009},
  publisher={SIAM}
}

@article{sorber2015structured,
  title={Structured data fusion},
  author={Sorber, Laurent and Van Barel, Marc and De Lathauwer, Lieven},
  journal={IEEE journal of selected topics in signal processing},
  volume={9},
  number={4},
  pages={586--600},
  year={2015},
  publisher={IEEE}
}

@article{chiantini2012generic,
  title={On generic identifiability of 3-tensors of small rank},
  author={Chiantini, Luca and Ottaviani, Giorgio},
  journal={SIAM Journal on Matrix Analysis and Applications},
  volume={33},
  number={3},
  pages={1018--1037},
  year={2012},
  publisher={SIAM}
}

@article{drton2011global,
  title={Global identifiability of linear structural equation models},
  author={Drton, Mathias and Foygel, Rina and Sullivant, Seth},
  journal={The Annals of Statistics},
  volume={39},
  number={2},
  pages={865--886},
  year={2011}
}

@inproceedings{squires2023linear,
  title={Linear causal disentanglement via interventions},
  author={Squires, Chandler and Seigal, Anna and Bhate, Salil S and Uhler, Caroline},
  booktitle={International conference on machine learning},
  pages={32540--32560},
  year={2023},
  organization={PMLR}
}

@article{shimizu2011directlingam,
  title={DirectLiNGAM: A Direct Method for Learning a Linear Non-{G}aussian Structural Equation Model},
  author={Shimizu, Shohei and Inazumi, Takanori and Sogawa, Yasuhiro and Hyvärinen, Aapo and Kawahara, Yoshinobu and Washio, Takashi and Hoyer, Patrik O. and Bollen, Kenneth},
  journal={Journal of Machine Learning Research},
  volume={12},
  pages={1225--1248},
  year={2011}
}

@article{chickering2002optimal,
  title={Optimal Structure Identification with Greedy Search},
  author={Chickering, David M.},
  journal={Journal of Machine Learning Research},
  volume={3},
  pages={507--554},
  year={2002}
}

@article{heurtebise2025identifiable,
  title={Identifiable multi-view causal discovery without non-gaussianity},
  author={Heurtebise, Ambroise and Chehab, Omar and Ablin, Pierre and Gramfort, Alexandre and Hyv{\"a}rinen, Aapo},
  journal={arXiv e-prints},
  pages={arXiv--2502},
  year={2025}
}

@article{hauser2012characterization,
  title={Characterization and Greedy Learning of Interventional {M}arkov Equivalence Classes of Directed Acyclic Graphs},
  author={Hauser, Alain and Bühlmann, Peter},
  journal={Journal of Machine Learning Research},
  volume={13},
  pages={2409--2464},
  year={2012}
}

@article{foygel2012half,
  title={Half-trek criterion for generic identifiability of linear structural equation models},
  author={Foygel, Rina and Draisma, Jan and Drton, Mathias},
  journal={The Annals of Statistics},
  pages={1682--1713},
  year={2012},
  publisher={JSTOR}
}

@inproceedings{zheng2018dags,
  title={{DAG}s with {NO} {TEARS}: Continuous Optimization for Structure Learning},
  author={Zheng, Xun and Aragam, Bryon and Ravikumar, Pradeep and Xing, Eric P.},
  booktitle={Advances in Neural Information Processing Systems},
  year={2018}
}

@article{reisach2021beware,
  title={Beware of the simulated {DAG}! causal discovery benchmarks may be easy to game},
  author={Reisach, Alexander and Seiler, Christof and Weichwald, Sebastian},
  journal={Advances in Neural Information Processing Systems},
  volume={34},
  pages={27772--27784},
  year={2021}
}

@article{peters2016causal,
  title={Causal Inference using Invariant Prediction: Identification and Confidence Intervals},
  author={Peters, Jonas and Bühlmann, Peter and Meinshausen, Nicolai},
  journal={Journal of the Royal Statistical Society: Series B},
  volume={78},
  number={5},
  pages={947--1012},
  year={2016}
}

@article{kruskal1977three,
  title={Three-way arrays: rank and uniqueness of trilinear decompositions, with application to arithmetic complexity and statistics},
  author={Kruskal, Joseph B},
  journal={Linear algebra and its applications},
  volume={18},
  number={2},
  pages={95--138},
  year={1977},
  publisher={Elsevier}
}

@article{student1908probable,
  title={The probable error of a mean},
  author={Student},
  journal={Biometrika},
  pages={1--25},
  year={1908},
  publisher={JSTOR}
}

@article{fisher1915frequency,
  title={Frequency distribution of the values of the correlation coefficient in samples from an indefinitely large population},
  author={Fisher, Ronald A},
  journal={Biometrika},
  volume={10},
  number={4},
  pages={507--521},
  year={1915},
  publisher={JSTOR}
}

@inproceedings{friedman2000using,
  title={Using {B}ayesian networks to analyze expression data},
  author={Friedman, Nir and Linial, Michal and Nachman, Iftach and Pe'er, Dana},
  booktitle={Proceedings of the fourth annual international conference on Computational molecular biology},
  pages={127--135},
  year={2000}
}

@article{sachs2005causal,
  title={Causal protein-signaling networks derived from multiparameter single-cell data},
  author={Sachs, Karen and Perez, Omar and Pe'er, Dana and Lauffenburger, Douglas A and Nolan, Garry P},
  journal={Science},
  volume={308},
  number={5721},
  pages={523--529},
  year={2005},
  publisher={American Association for the Advancement of Science}
}

@article{siddiqi2022causal,
  title={Causal mapping of human brain function},
  author={Siddiqi, Shan H and Kording, Konrad P and Parvizi, Josef and Fox, Michael D},
  journal={Nature reviews neuroscience},
  volume={23},
  number={6},
  pages={361--375},
  year={2022},
  publisher={Nature Publishing Group UK London}
}

@article{imbens2020potential,
  title={Potential outcome and directed acyclic graph approaches to causality: Relevance for empirical practice in economics},
  author={Imbens, Guido W},
  journal={Journal of Economic Literature},
  volume={58},
  number={4},
  pages={1129--1179},
  year={2020},
  publisher={American Economic Association 2014 Broadway, Suite 305, Nashville, TN 37203-2425}
}

@article{greenland1999causal,
  title={Causal diagrams for epidemiologic research},
  author={Greenland, Sander and Pearl, Judea and Robins, James M},
  journal={Epidemiology},
  volume={10},
  number={1},
  pages={37--48},
  year={1999},
  publisher={LWW}
}

@article{scholkopf2021toward,
  title={Toward causal representation learning},
  author={Sch{\"o}lkopf, Bernhard and Locatello, Francesco and Bauer, Stefan and Ke, Nan Rosemary and Kalchbrenner, Nal and Goyal, Anirudh and Bengio, Yoshua},
  journal={Proceedings of the IEEE},
  volume={109},
  number={5},
  pages={612--634},
  year={2021},
  publisher={IEEE}
}

@article{wang2026multi,
  title={Multi-context principal component analysis},
  author={Wang, Kexin and Bhate, Salil and Pereira, Jo{\~a}o M and Kileel, Joe and Figlerowicz, Matylda and Seigal, Anna},
  journal={arXiv preprint arXiv:2601.15239},
  year={2026}
}

@article{meinshausen2016methods,
  title={Methods for causal inference from gene perturbation experiments and validation},
  author={Meinshausen, Nicolai and Hauser, Alain and Mooij, Joris M and Peters, Jonas and Versteeg, Philip and B{\"u}hlmann, Peter},
  journal={Proceedings of the National Academy of Sciences},
  volume={113},
  number={27},
  pages={7361--7368},
  year={2016},
  publisher={National Academy of Sciences}
}

@inproceedings{itani2010structure,
  title={Structure learning in causal cyclic networks},
  author={Itani, Sleiman and Ohannessian, Mesrob and Sachs, Karen and Nolan, Garry P and Dahleh, Munther A},
  booktitle={Causality: Objectives and assessment},
  pages={165--176},
  year={2010},
  organization={PMLR}
}

@inproceedings{verma2022equivalence,
  title={Equivalence and synthesis of causal models},
  author={Verma, Thomas and Pearl, Judea},
  booktitle={Proceedings of the Sixth Annual Conference on Uncertainty in Artificial Intelligence},
  pages={255--270},
  year={1990}
}

@article{wang2017permutation,
  title={Permutation-based causal inference algorithms with interventions},
  author={Wang, Yuhao and Solus, Liam and Yang, Karren and Uhler, Caroline},
  journal={Advances in neural information processing systems},
  volume={30},
  year={2017}
}

@article{schultheiss2023ancestor,
  title={Ancestor regression in linear structural equation models},
  author={Schultheiss, Christoph and B{\"u}hlmann, Peter},
  journal={Biometrika},
  volume={110},
  number={4},
  pages={1117--1124},
  year={2023},
  publisher={Oxford University Press}
}

@article{rothenhausler2015backshift,
  title={BACKSHIFT: Learning causal cyclic graphs from unknown shift interventions},
  author={Rothenh{\"a}usler, Dominik and Heinze, Christina and Peters, Jonas and Meinshausen, Nicolai},
  journal={Advances in neural information processing systems},
  volume={28},
  year={2015}
}

@book{pearl2009causality,
  title={Causality},
  author={Pearl, Judea},
  edition={2},
  year={2009},
  publisher={Cambridge University Press}
}

@inproceedings{markowetz2005probabilistic,
  title={Probabilistic soft interventions in conditional {G}aussian networks},
  author={Markowetz, Florian and Grossmann, Steffen and Spang, Rainer},
  booktitle={International Workshop on Artificial Intelligence and Statistics},
  pages={214--221},
  year={2005},
  organization={PMLR}
}

@article{gamella2025causal,
  title={Causal chambers as a real-world physical testbed for {AI} methodology},
  author={Gamella, Juan L and Peters, Jonas and B{\"u}hlmann, Peter},
  journal={Nature Machine Intelligence},
  volume={7},
  number={1},
  pages={107--118},
  year={2025},
  publisher={Nature Publishing Group UK London}
}

@article{brito2002new,
  title={A new identification condition for recursive models with correlated errors},
  author={Brito, Carlos and Pearl, Judea},
  journal={Structural Equation Modeling},
  volume={9},
  number={4},
  pages={459--474},
  year={2002},
  publisher={Taylor \& Francis}
}

@book{sullivant2023algebraic,
  title={Algebraic statistics},
  author={Sullivant, Seth},
  volume={194},
  year={2023},
  publisher={American Mathematical Society}
}

@book{wedderburn1934lectures,
  title={Lectures on matrices},
  author={Wedderburn, Joseph Henry Maclagan},
  volume={17},
  year={1934},
  publisher={American Mathematical Soc.}
}

@article{kileel2025subspace,
  title={Subspace power method for symmetric tensor decomposition},
  author={Kileel, Joe and Pereira, Jo{\~a}o M},
  journal={Numerical Algorithms},
  pages={1--38},
  year={2025},
  publisher={Springer}
}

@article{ranestad2026real,
  title={A real generalized trisecant trichotomy},
  author={Ranestad, Kristian and Seigal, Anna and Wang, Kexin},
  journal={Journal of Algebra},
  year={2026},
  publisher={Elsevier}
}


\appendix

\section{Algorithm Details}

\subsection{Hypothesis testing for zero correlation}\label{app:hypothesis}

The pairwise intervention-asymmetric tests whether a sample correlation is effectively zero. Finite sample can make small correlations appear nonzero (and vice versa), so we use a three-way classification:
\[
\texttt{NONZERO}, \quad \texttt{ZERO}, \quad \texttt{INCONCLUSIVE}.
\]
This allows us to avoid overconfident conclusions when the data are ambiguous.

\paragraph{\em Testing for nonzero correlation.}
Given a sample correlation \(\hat\rho\) from \(N\) samples, we test
\[
H_0:\rho=0
\qquad\text{against}\qquad
H_1:\rho\neq 0.
\]
We use the standard \(t\)-statistic
\[
t=\hat\rho\sqrt{\frac{N-2}{1-\hat\rho^2}},
\]
which follows a \(t\)-distribution with \(N-2\) degrees of freedom under \(H_0\)~\cite{student1908probable}. If the two-sided \(p\)-value is less than a significance level \(\alpha\), we classify the correlation as \texttt{NONZERO}.

\paragraph{\em Testing for effective zero correlation.}
If the null is not rejected, we next test the correlation using an equivalence test. Fix a tolerance \(\epsilon>0\), and test
\[
H_0: |\rho|\geq \epsilon
\qquad\text{against}\qquad
H_1: |\rho|<\epsilon.
\]
We implement this using two one-sided tests. Using Fisher's \(z\)-transform~\cite{fisher1915frequency},
\[
z=\operatorname{arctanh}(\hat\rho),
\qquad
z_\pm=\operatorname{arctanh}(\pm \epsilon),
\]
with standard error
\[
\mathrm{SE}=\frac{1}{\sqrt{N-3}},
\]
we compute
\[
Z_-=\frac{z-z_-}{\mathrm{SE}},
\qquad
Z_+=\frac{z-z_+}{\mathrm{SE}}.
\]
If
\(
Z_->z_{1-\alpha}\)
and \(
Z_+<-z_{1-\alpha}\),
we reject \(H_0\) and classify the correlation as \texttt{ZERO}.

\paragraph{\em Inconclusive cases.}
If neither test is decisive, we classify the correlation as \texttt{INCONCLUSIVE}. This occurs when the data cannot distinguish between zero and nonzero correlation.

\paragraph{\em Summary.}
We classify each sample correlation as follows:
\begin{itemize}
    \item \texttt{NONZERO}: reject \(H_0:\rho=0\),
    \item \texttt{ZERO}: accept equivalence \(|\rho|<\epsilon\),
    \item \texttt{INCONCLUSIVE}: neither test is decisive.
\end{itemize}
This classification is used in the pairwise intervention-asymmetric test to infer causal order.

\subsection{From causal order to adjacency matrix}\label{sec:estimate Lambda}

Given a causal order \(\pi = (\pi_1,\ldots,\pi_p)\), recovering the adjacency matrix \(\Lambda\) reduces to estimating the coefficients of each node from its ancestors in the order. For each node \(\pi_t\), its parents lie in \(\{\pi_1,\ldots,\pi_{t-1}\}\), so estimation can be restricted to this set.
We consider the following approaches. 

\paragraph{\em Lasso regression.}
For each node \(\pi_t\), we regress \(x_{\pi_t}\) onto its ancestors \(\{x_{\pi_1},\ldots,x_{\pi_{t-1}}\}\) using Lasso. The \(\ell_1\)-penalty encourages sparse solutions. After identifying the support, one can optionally refit an ordinary least squares regression restricted to the parents to obtain less biased estimates. The resulting coefficients define the corresponding row of \(\Lambda\).

\paragraph{\em Regression.}
For each \(\pi_t\), we regress \(x_{\pi_t}\) onto its ancestors. The coefficients define the corresponding row of \(\Lambda\). We set entries with magnitude below a threshold to zero.

\paragraph{\em Cholesky decomposition.}
The causal order induces a permutation under which the matrix $(I - \Lambda)$ is lower triangular. Applying a Cholesky decomposition to the permuted precision matrix yields a dense estimate of the structural coefficients. To obtain a sparse adjacency matrix, we set entries with magnitude below a threshold to zero.

The Lasso approach requires access to individual samples; the other two methods do not.

\section{Experiment details}\label{app:exp_details}

\subsection{Linear SEM} \label{app:LSEM}

For each experiment instance, we generate a random DAG with \(p=10\) nodes. We first sample a weighted adjacency matrix \(\Lambda \in \mathbb{R}^{p \times p}\), where each edge is included independently with probability \(0.6\). Nonzero edge weights are drawn uniformly from \([0.4,1]\) with random signs. To enforce acyclicity, we construct \(\Lambda\) to be lower-triangular and then apply a random permutation.

Noise variables \(\epsilon_i\) are independent across nodes, with variances sampled i.i.d.\ from \(\mathrm{U}[0.1,1]\). We consider three noise settings:
(i) fully Gaussian noise,
(ii) half of the nodes have Gaussian noise and half follow a Student-\(t_5\) distribution,
and (iii) fully Student-\(t_5\) noise.

We use the intervention-pattern matrix 
\[
B =
\begin{pmatrix}
1 & 1 & 1 & 1 & 1 \\
1 & 1 & 1 & 1 & 0 \\
1 & 1 & 1 & 0 & 1 \\
1 & 1 & 1 & 0 & 0 \\
1 & 1 & 0 & 1 & 1 \\
1 & 1 & 0 & 1 & 0 \\
1 & 1 & 0 & 0 & 1 \\
1 & 1 & 0 & 0 & 0 \\
1 & 0 & 1 & 1 & 1 \\
1 & 0 & 1 & 1 & 0
\end{pmatrix} \in \RR^{10 \times 5}.
\]
Each row gives the intervention-pattern vector across the $5$ contexts.

For each context \(j\), we generate samples from the corresponding interventional LSEM. Under perfect interventions, all incoming edges to intervened nodes are removed and their structural equations are replaced with exogenous noise. Samples are generated by solving the linear system
\[
(I - \Lambda^{(j)}) X^{(j)} = E^{(j)},
\]
where \(\Lambda^{(j)}\) is the modified adjacency matrix for context \(j\), and \(E^{(j)}\) contains independent noise samples.
Each context contains the same number of samples, ranging from \(100\) to \(10{,}000\). Each experimental configuration is repeated over 30 independent trials.
The root is selected using a candidate set of size $3$ (Algorithm \ref{alg:pairwise_asym_corr}). 
From the causal order, $\Lambda$ is estimated using Lasso regression.

\subsection{Scalability}
\label{app:scalability_details}

We vary the number of nodes \(p \in \{50,60,\dots,400\}\). For each \(p\), we generate a random DAG with edge probability \(6/(p-1)\), yielding an expected in-degree of around \(3\). Edge weights are drawn from \(\mathrm{Unif}([-1,-0.1]\cup[0.1,1])\), and noise variances from \(\mathrm{Unif}(0.01,0.05)\), with Gaussian noise.
We use a binary-code design with \(L=\lceil \log_2 p \rceil + 1\) contexts, including one observational context. Each context contains \(p^2\) samples.
The root is selected using a candidate set of size $3$ (Algorithm \ref{alg:pairwise_asym_corr}). The experimental configuration was repeated over 10 independent trials. Given the recovered causal order, we estimate the adjacency matrix via regression in the observational context with threshold \(0.1\).
Performance is evaluated using the F1 score on the recovered adjacency support.

\subsection{Nonlinear SEM}

We consider nonlinear SEMs with perfect interventions on random DAGs with \(p=10\) nodes and edge probability \(0.8\).
Data are generated by
\[
x_i = f_i(x_{\mathrm{pa}(i)}) + \epsilon_i,
\]
where \(f_i\) is a randomly generated nonlinear function and \(\epsilon_i\) is Gaussian noise. Each \(f_i\) is implemented as a one-hidden-layer neural network with randomly initialized weights and a combination of activation functions, including tanh, ReLU, and square activations.

Under interventions, if a node appears as a parent in a downstream structural equation, its input is masked to zero in that context, removing its influence. The structural equation of an intervened node is replaced by an exogenous noise term.

We construct the causal order greedily. The initial root is selected based on stability of marginal variances across contexts where the node is not intervened on. At each subsequent step, for each candidate node, we train a feedforward neural network to predict the node from previously selected variables, using data pooled across contexts where the node is not intervened on. If a parent variable is intervened on in a context, its input is set to zero.
To estimate the parent set, we use a gated neural network, where gates act as feature-selection weights. Variables with gate logits above a threshold are retained as parents. The model is then optionally refit using only the selected parents; if this refit improves validation loss, the reduced parent set is accepted. Otherwise, the full previously selected set is kept.

Node selection is based on stability of prediction error across contexts. For each candidate node, we compute the mean squared error (MSE) in each valid context and select the node with the most stable normalized MSE, measured by the variance across contexts with a small penalty on the mean.

\subsection{Flow-cytometry Dataset}

We evaluate our method on the flow-cytometry dataset \cite{sachs2005causal}, which records single-cell protein phosphorylation.
for $p = 11$ signaling proteins across nine experimental conditions. We (i) pool the two unperturbed baselines (anti-CD3/CD28 with and without ICAM-2 co-stimulation) into a single observational environment, and (ii) restrict the intervention contexts to the five kinase-inhibitor conditions, each targeting a single protein (Akt, PKC, PIP2, Mek, PIP3 via PI3K). 
We exclude the two pharmacological-activator conditions (PMA and $\beta$2-cAMP), because inhibitors more closely approximate a perfect intervention while for activators, the target remains conditionally dependent on its parents. 
This yields six environments and $5{,}846$ cells in total. 
We compare the graph recovered by our method (Table~\ref{tab:sachs-edges}) with two reference DAGs from \cite{sachs2005causal} and the recovered graphs of ICP and hiddenICP~\cite{peters2016causal,meinshausen2016methods}.

\subsection{Light Causal Chamber}

We use the Causal Chamber light-tunnel dataset with single node interventions \texttt{lt\_interventions\_standard\_v1} from \cite{gamella2025causal}. The selected variables and dataset columns are
\[
\{\texttt{red}, \texttt{green}, \texttt{blue}, \texttt{ir\_1}, \texttt{vis\_1},
\texttt{l\_11}, \texttt{l\_12}, \texttt{diode\_vis\_1}, \texttt{t\_vis\_1}\},
\]
corresponding to
\(
R, G, B, \tilde I_1, \tilde V_1, L_{11}, L_{12}, D^V_1, T^V_1\).
The variables \(R,G,B\) control the main red, green, and blue LEDs, while \(L_{11},L_{12}\) control the two auxiliary LEDs placed by the first light-intensity sensor. The variables \(D^V_1\) and \(T^V_1\) are visible-light sensor parameters for the first sensor, controlling the selected visible photodiode and the photodiode exposure time. The variables \(\tilde I_1\) and \(\tilde V_1\) are sensor readings: the infrared and visible-light intensity measurements from the first sensor.

The selected ground-truth graph contains twelve directed edges:
\[
R,G,B,L_{11},L_{12} \to \tilde I_1,
\qquad
R,G,B,L_{11},L_{12},D^V_1,T^V_1 \to \tilde V_1.
\]
We use the reference experiment together with interventions on the selected actuator and sensor-parameter variables:
\[
\begin{aligned}
&\texttt{uniform\_reference},\ 
\texttt{uniform\_red\_strong},\ 
\texttt{uniform\_green\_strong},\\ 
&\texttt{uniform\_blue\_strong}, \
\texttt{uniform\_t\_vis\_1\_strong},\ 
\texttt{uniform\_l\_11\_mid},\\
&\texttt{uniform\_l\_12\_mid},\ 
\texttt{uniform\_diode\_vis\_1\_mid}.
\end{aligned}
\]

TSCD recovers the six edges from \(R,G,B\) to \(\tilde I_1\) and \(\tilde V_1\), corresponding to the strongest approximately linear effects in the subsystem. 
For TSCD-nonlinear, we set gate threshold to \(2\). It recovers all twelve ground-truth edges and adds three false positives, giving precision \(12/15=0.800\) and recall \(12/12=1.000\).

For the other methods, we add a small independent Gaussian noise \(10^{-5}\epsilon\), with \(\epsilon \sim \mathcal{N}(0,1)\), to each data entry before running the algorithms. This is necessary because some variables have zero empirical variance in certain intervention context, which causes numerical failures for several baselines. The perturbation is negligible relative to the scale of the measurements and is used only to help the other methods.

\subsection{Root Projection Scores}
\label{app:root projection scores}

We study the projection score of roots, specifically, where the true roots appear in the ranking of nodes by projection score. In the population setting, the roots would all have projection score equal to $1$ and appear at the top, but in practice they can be lower down the list. How far down they may appear guides the number of candidates that should be used for candidate root selection in Section~\ref{sec: proj_norm}. The top scoring root has a better rank when the graph is more sparse, since there are more roots.

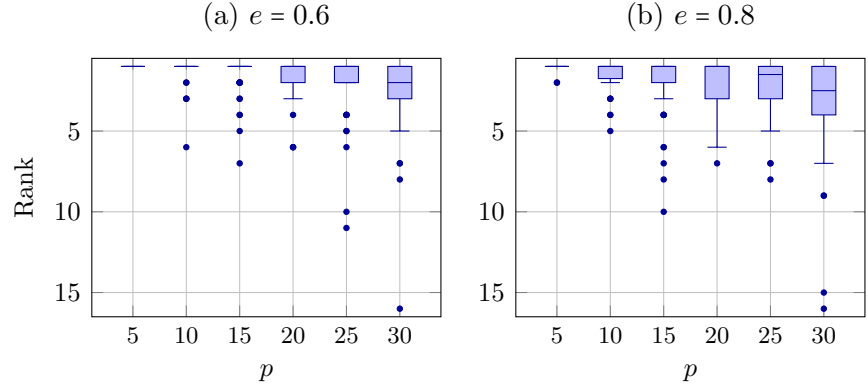
\begin{figure}[htbp]
    \centering

\begin{tikzpicture}

\begin{axis}[
    name=leftplot,
    width=6.2cm, height=5.0cm,
    title={(a) $e=0.6$},
    ylabel={Rank},
    xlabel={$p$},
    ymin=0.5, ymax=16.5,
    y dir=reverse,
    grid=major,
    grid style={line width=.1pt, draw=gray!30},
    major grid style={line width=.2pt, draw=gray!50},
    tick label style={font=\footnotesize},
    label style={font=\footnotesize},
    title style={font=\small},
    xtick={1, 2, 3, 4, 5, 6},
    xticklabels={$5$, $10$, $15$, $20$, $25$, $30$},
    x tick label style={font=\scriptsize},
    boxplot/draw direction=y,
    boxplot/box extend=0.45,
]

\addplot+[
    boxplot prepared={
        draw position=1,
        lower whisker=1,
        lower quartile=1,
        median=1,
        upper quartile=1,
        upper whisker=1
    },
    fill=blue!25,
    draw=blue!60!black,
    solid,
    mark=*,
    mark size=1pt,
    mark options={fill=blue!60!black, draw=blue!60!black, solid}
] coordinates {};

\addplot+[
    boxplot prepared={
        draw position=2,
        lower whisker=1,
        lower quartile=1,
        median=1,
        upper quartile=1,
        upper whisker=1
    },
    fill=blue!25,
    draw=blue!60!black,
    solid,
    mark=*,
    mark size=1pt,
    mark options={fill=blue!60!black, draw=blue!60!black, solid}
] coordinates {(2,2) (2,2) (2,3) (2,3) (2,3) (2,6)};

\addplot+[
    boxplot prepared={
        draw position=3,
        lower whisker=1,
        lower quartile=1,
        median=1,
        upper quartile=1,
        upper whisker=1
    },
    fill=blue!25,
    draw=blue!60!black,
    solid,
    mark=*,
    mark size=1pt,
    mark options={fill=blue!60!black, draw=blue!60!black, solid}
] coordinates {(3,2) (3,2) (3,2) (3,2) (3,2) (3,2) (3,3) (3,3) (3,4) (3,4) (3,5) (3,7)};

\addplot+[
    boxplot prepared={
        draw position=4,
        lower whisker=1,
        lower quartile=1,
        median=1,
        upper quartile=2,
        upper whisker=3
    },
    fill=blue!25,
    draw=blue!60!black,
    solid,
    mark=*,
    mark size=1pt,
    mark options={fill=blue!60!black, draw=blue!60!black, solid}
] coordinates {(4,4) (4,6) (4,6) (4,6)};

\addplot+[
    boxplot prepared={
        draw position=5,
        lower whisker=1,
        lower quartile=1,
        median=1,
        upper quartile=2,
        upper whisker=2
    },
    fill=blue!25,
    draw=blue!60!black,
    solid,
    mark=*,
    mark size=1pt,
    mark options={fill=blue!60!black, draw=blue!60!black, solid}
] coordinates {(5,4) (5,4) (5,4) (5,4) (5,5) (5,5) (5,6) (5,10) (5,11)};

\addplot+[
    boxplot prepared={
        draw position=6,
        lower whisker=1,
        lower quartile=1,
        median=2,
        upper quartile=3,
        upper whisker=5
    },
    fill=blue!25,
    draw=blue!60!black,
    solid,
    mark=*,
    mark size=1pt,
    mark options={fill=blue!60!black, draw=blue!60!black, solid}
] coordinates {(6,7) (6,7) (6,8) (6,16)};

\end{axis}

\begin{axis}[
    name=rightplot,
    at={(leftplot.outer east)},
    anchor=outer west,
    xshift=0.35cm,
    width=6.2cm, height=5.0cm,
    title={(b) $e=0.8$},
    ylabel={},
    xlabel={$p$},
    ymin=0.5, ymax=16.5,
    y dir=reverse,
    grid=major,
    grid style={line width=.1pt, draw=gray!30},
    major grid style={line width=.2pt, draw=gray!50},
    tick label style={font=\footnotesize},
    label style={font=\footnotesize},
    title style={font=\small},
    xtick={1, 2, 3, 4, 5, 6},
    xticklabels={$5$, $10$, $15$, $20$, $25$, $30$},
    x tick label style={font=\scriptsize},
    boxplot/draw direction=y,
    boxplot/box extend=0.45,
]

\addplot+[
    boxplot prepared={
        draw position=1,
        lower whisker=1,
        lower quartile=1,
        median=1,
        upper quartile=1,
        upper whisker=1
    },
    fill=blue!25,
    draw=blue!60!black,
    solid,
    mark=*,
    mark size=1pt,
    mark options={fill=blue!60!black, draw=blue!60!black, solid}
] coordinates {(1,2) (1,2)};

\addplot+[
    boxplot prepared={
        draw position=2,
        lower whisker=1,
        lower quartile=1,
        median=1,
        upper quartile=1.75,
        upper whisker=2
    },
    fill=blue!25,
    draw=blue!60!black,
    solid,
    mark=*,
    mark size=1pt,
    mark options={fill=blue!60!black, draw=blue!60!black, solid}
] coordinates {(2,3) (2,3) (2,3) (2,4) (2,4) (2,5)};

\addplot+[
    boxplot prepared={
        draw position=3,
        lower whisker=1,
        lower quartile=1,
        median=1,
        upper quartile=2,
        upper whisker=3
    },
    fill=blue!25,
    draw=blue!60!black,
    solid,
    mark=*,
    mark size=1pt,
    mark options={fill=blue!60!black, draw=blue!60!black, solid}
] coordinates {(3,4) (3,4) (3,4) (3,4) (3,6) (3,6) (3,7) (3,8) (3,10)};

\addplot+[
    boxplot prepared={
        draw position=4,
        lower whisker=1,
        lower quartile=1,
        median=1,
        upper quartile=3,
        upper whisker=6
    },
    fill=blue!25,
    draw=blue!60!black,
    solid,
    mark=*,
    mark size=1pt,
    mark options={fill=blue!60!black, draw=blue!60!black, solid}
] coordinates {(4,7)};

\addplot+[
    boxplot prepared={
        draw position=5,
        lower whisker=1,
        lower quartile=1,
        median=1.5,
        upper quartile=3,
        upper whisker=5
    },
    fill=blue!25,
    draw=blue!60!black,
    solid,
    mark=*,
    mark size=1pt,
    mark options={fill=blue!60!black, draw=blue!60!black, solid}
] coordinates {(5,7) (5,7) (5,8)};

\addplot+[
    boxplot prepared={
        draw position=6,
        lower whisker=1,
        lower quartile=1,
        median=2.5,
        upper quartile=4,
        upper whisker=7
    },
    fill=blue!25,
    draw=blue!60!black,
    solid,
    mark=*,
    mark size=1pt,
    mark options={fill=blue!60!black, draw=blue!60!black, solid}
] coordinates {(6,9) (6,9) (6,15) (6,16)};

\end{axis}

\end{tikzpicture}
    \caption{The distribution of best ranks for root nodes in projection score, as 
  $p$ (the number of nodes) varies, for two values of $e$ (the edge probability), cf. Section~\ref{app:LSEM}.
  All experiments were conducted with 1000 samples, half of the nodes given Gaussian noise and the rest given Student-$t_5$ distribution noise.
  }
  \label{fig:root projection scores}
\end{figure}

\subsection{Ablation}
\label{app:ablation}
We perform an ablation study using the LSEM data generation setup and hyperparameters in Section~\ref{app:LSEM}. At each step, TSCD selects a candidate set of possible root nodes, then tests pairwise correlations to make a choice of root. We compare it to two ablated versions. The first only uses projection norm as a score to select the root at each step, i.e. there is only one root candidate.
The second only uses pairwise correlation testing to determine the root at each step, i.e. all nodes are root candidates. 

Figure~\ref{fig:ablation} reports these results. TSCD has improved performance compared to the ablated algorithms at low to moderate sample sizes. This justifies the importance of including both steps: selecting root candidates using projection norm and using pairwise correlations to choose the root from the candidates.

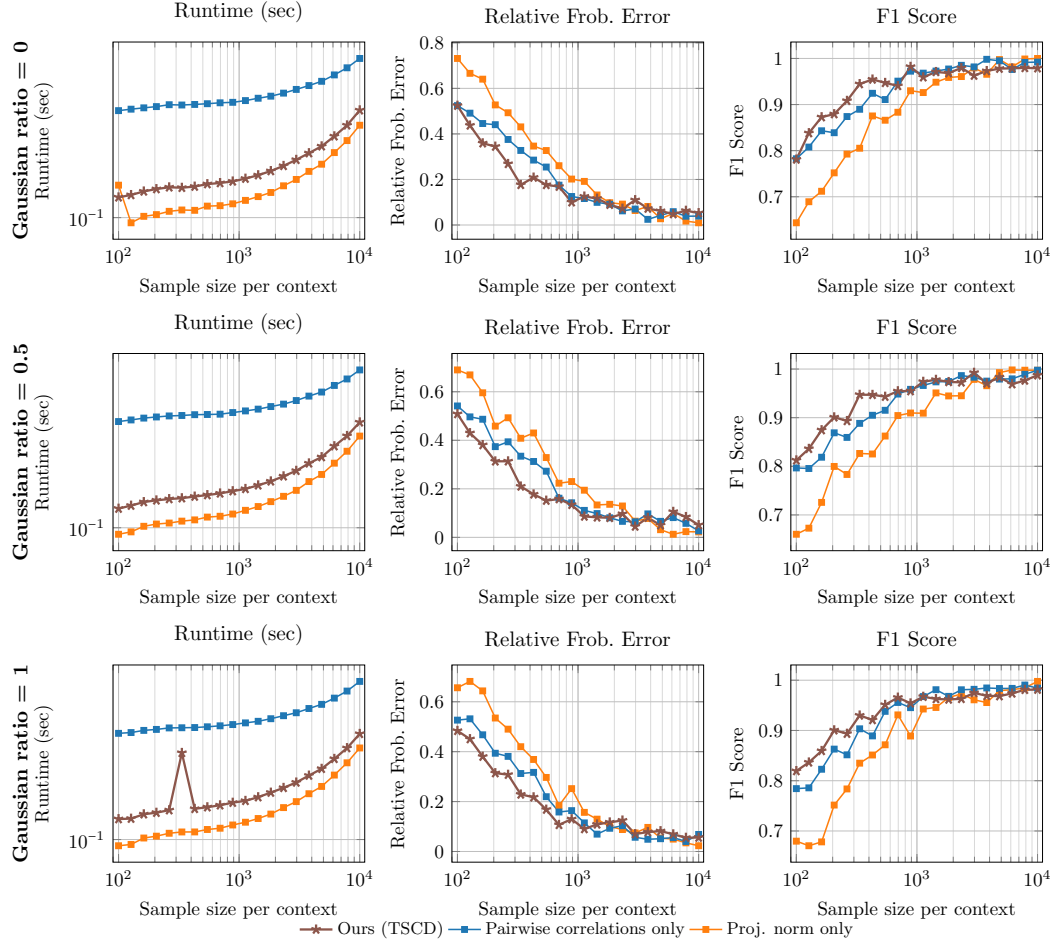
\begin{figure}[bp]
  \centering
  \resizebox{0.85\linewidth}{!}{
\begin{tikzpicture}
\begin{groupplot}[
    group style={group size=3 by 1, horizontal sep=1.6cm},
    width=6.2cm, height=5.2cm,
    xmode=log, log basis x=10,
    xmin=90, xmax=11000,
    xlabel={Sample size per context},
    grid=both, grid style={line width=.1pt, draw=gray!30},
    major grid style={line width=.2pt, draw=gray!50},
    tick label style={font=\footnotesize},
    label style={font=\footnotesize},
    title style={font=\small},
    legend style={font=\scriptsize, draw=none, fill=none},
    legend cell align=left,
]
\nextgroupplot[
    title={Runtime (sec)},
    ylabel={Runtime (sec)},
    ymode=log, log basis y=10,
ytick={0.01, 0.1, 1},
    minor y tick num=0,
]
\addplot[color={rgb,1:red,1.000000;green,0.498039;blue,0.054902}, mark=square*, mark size=1.4pt, line width=0.8pt] coordinates {(100,0.12669847) (127,0.09629868) (162,0.10094557) (206,0.10221074) (263,0.10470151) (335,0.10572395) (428,0.1054622) (545,0.108753) (695,0.10909236) (885,0.11059021) (1128,0.11339255) (1438,0.11654429) (1832,0.12012738) (2335,0.12609681) (2976,0.13209809) (3792,0.14008609) (4832,0.14761604) (6158,0.16065065) (7847,0.17536166) (10000,0.19613596)};
\addplot[color={rgb,1:red,0.121569;green,0.466667;blue,0.705882}, mark=square*, mark size=1.4pt, line width=0.8pt] coordinates {(100,0.21814863) (127,0.22048082) (162,0.22263884) (206,0.22464362) (263,0.22777563) (335,0.22727818) (428,0.22821761) (545,0.22942188) (695,0.23097131) (885,0.23176763) (1128,0.23470931) (1438,0.23877479) (1832,0.24238573) (2335,0.24783855) (2976,0.25464063) (3792,0.26187467) (4832,0.27014812) (6158,0.28303118) (7847,0.29813139) (10000,0.31913434)};
\addplot[color={rgb,1:red,0.549020;green,0.337255;blue,0.294118}, mark=star, mark size=2.6pt, line width=1.2pt] coordinates {(100,0.11595316) (127,0.11796435) (162,0.12105814) (206,0.12317386) (263,0.12475445) (335,0.12434799) (428,0.12542122) (545,0.12770933) (695,0.12871637) (885,0.13022197) (1128,0.13291296) (1438,0.13627199) (1832,0.14046937) (2335,0.14586346) (2976,0.15244985) (3792,0.16005078) (4832,0.16826549) (6158,0.18098962) (7847,0.19604329) (10000,0.21873396)};
\nextgroupplot[
    title={Relative Frob.\ Error},
    ylabel={Relative Frob.\ Error},
]
\addplot[color={rgb,1:red,1.000000;green,0.498039;blue,0.054902}, mark=square*, mark size=1.4pt, line width=0.8pt] coordinates {(100,0.73139261) (127,0.66604103) (162,0.64048811) (206,0.52734554) (263,0.49293629) (335,0.43091542) (428,0.34664514) (545,0.32701225) (695,0.26042868) (885,0.20172062) (1128,0.19149836) (1438,0.13186636) (1832,0.09795992) (2335,0.09176342) (2976,0.06359648) (3792,0.08152757) (4832,0.02749625) (6158,0.05052777) (7847,0.01713123) (10000,0.00971442)};
\addplot[color={rgb,1:red,0.121569;green,0.466667;blue,0.705882}, mark=square*, mark size=1.4pt, line width=0.8pt] coordinates {(100,0.52423472) (127,0.49113365) (162,0.44565076) (206,0.44057961) (263,0.37565955) (335,0.32744746) (428,0.2850568) (545,0.25451866) (695,0.17500215) (885,0.12517185) (1128,0.11612685) (1438,0.09939223) (1832,0.09104526) (2335,0.06203653) (2976,0.07019744) (3792,0.02484194) (4832,0.04192715) (6158,0.05861001) (7847,0.03855969) (10000,0.03907097)};
\addplot[color={rgb,1:red,0.549020;green,0.337255;blue,0.294118}, mark=star, mark size=2.6pt, line width=1.2pt] coordinates {(100,0.52415903) (127,0.43734642) (162,0.35913161) (206,0.34454038) (263,0.26845624) (335,0.17654234) (428,0.20858341) (545,0.17624091) (695,0.16855312) (885,0.1000925) (1128,0.12532953) (1438,0.11622441) (1832,0.0891938) (2335,0.07220481) (2976,0.10952329) (3792,0.0722964) (4832,0.06172745) (6158,0.04797124) (7847,0.06293443) (10000,0.05257324)};
\nextgroupplot[
    title={F1 Score},
    ylabel={F1 Score},
]
\addplot[color={rgb,1:red,1.000000;green,0.498039;blue,0.054902}, mark=square*, mark size=1.4pt, line width=0.8pt] coordinates {(100,0.64352209) (127,0.68921805) (162,0.71190393) (206,0.75183236) (263,0.79255797) (335,0.80525103) (428,0.87529973) (545,0.86598912) (695,0.88328388) (885,0.92996337) (1128,0.92600823) (1438,0.94846403) (1832,0.9587941) (2335,0.96071254) (2976,0.97950186) (3792,0.96564086) (4832,0.99759702) (6158,0.98166907) (7847,0.99880952) (10000,1)};
\addplot[color={rgb,1:red,0.121569;green,0.466667;blue,0.705882}, mark=square*, mark size=1.4pt, line width=0.8pt] coordinates {(100,0.7824286) (127,0.80776326) (162,0.84312376) (206,0.83889323) (263,0.87398689) (335,0.88972721) (428,0.92439684) (545,0.91086372) (695,0.95078817) (885,0.97239598) (1128,0.96813801) (1438,0.9726431) (1832,0.97710856) (2335,0.98510997) (2976,0.98167293) (3792,0.99812925) (4832,0.99486478) (6158,0.97607663) (7847,0.99146199) (10000,0.9920813)};
\addplot[color={rgb,1:red,0.549020;green,0.337255;blue,0.294118}, mark=star, mark size=2.6pt, line width=1.2pt] coordinates {(100,0.7809014) (127,0.83812754) (162,0.87268299) (206,0.87930229) (263,0.90820287) (335,0.94451894) (428,0.95447765) (545,0.94711685) (695,0.94045082) (885,0.98197652) (1128,0.95931994) (1438,0.97078479) (1832,0.96817104) (2335,0.97947533) (2976,0.96286167) (3792,0.97287908) (4832,0.97795687) (6158,0.97864599) (7847,0.97950228) (10000,0.97881909)};
\end{groupplot}
\node[font=\small\bfseries, rotate=90]
at ($(group c1r1.west) + (-1.7cm,0)$)
{Gaussian ratio = 0};
\end{tikzpicture}}
  \resizebox{0.85\linewidth}{!}{
\begin{tikzpicture}
\begin{groupplot}[
    group style={group size=3 by 1, horizontal sep=1.6cm},
    width=6.2cm, height=5.2cm,
    xmode=log, log basis x=10,
    xmin=90, xmax=11000,
    xlabel={Sample size per context},
    grid=both, grid style={line width=.1pt, draw=gray!30},
    major grid style={line width=.2pt, draw=gray!50},
    tick label style={font=\footnotesize},
    label style={font=\footnotesize},
    title style={font=\small},
    legend style={font=\scriptsize, draw=none, fill=none},
    legend cell align=left,
]
\nextgroupplot[
    title={Runtime (sec)},
    ylabel={Runtime (sec)},
    ymode=log, log basis y=10,
ytick={0.01, 0.1, 1},
    minor y tick num=0,
]
\addplot[color={rgb,1:red,1.000000;green,0.498039;blue,0.054902}, mark=square*, mark size=1.4pt, line width=0.8pt] coordinates {(100,0.09534571) (127,0.09703499) (162,0.10108747) (206,0.10276444) (263,0.10351726) (335,0.10498289) (428,0.1059449) (545,0.10813547) (695,0.10879488) (885,0.11058754) (1128,0.11363468) (1438,0.11682088) (1832,0.1211019) (2335,0.12591349) (2976,0.13154517) (3792,0.14028277) (4832,0.1478329) (6158,0.16054182) (7847,0.17522014) (10000,0.19581931)};
\addplot[color={rgb,1:red,0.121569;green,0.466667;blue,0.705882}, mark=square*, mark size=1.4pt, line width=0.8pt] coordinates {(100,0.2178033) (127,0.22038114) (162,0.22339338) (206,0.22525074) (263,0.22695448) (335,0.22769259) (428,0.22904496) (545,0.22935526) (695,0.22986681) (885,0.23263707) (1128,0.23546773) (1438,0.23856219) (1832,0.24308485) (2335,0.24748637) (2976,0.25403848) (3792,0.26246441) (4832,0.27040336) (6158,0.28381223) (7847,0.29761298) (10000,0.31795843)};
\addplot[color={rgb,1:red,0.549020;green,0.337255;blue,0.294118}, mark=star, mark size=2.6pt, line width=1.2pt] coordinates {(100,0.11507375) (127,0.11764945) (162,0.12075289) (206,0.12225911) (263,0.12344045) (335,0.12421138) (428,0.12573008) (545,0.12709689) (695,0.12865632) (885,0.13078979) (1128,0.1329831) (1438,0.13650754) (1832,0.14053851) (2335,0.14586016) (2976,0.1518884) (3792,0.16037038) (4832,0.16788922) (6158,0.1818839) (7847,0.19610956) (10000,0.21659647)};
\nextgroupplot[
    title={Relative Frob.\ Error},
    ylabel={Relative Frob.\ Error},
]
\addplot[color={rgb,1:red,1.000000;green,0.498039;blue,0.054902}, mark=square*, mark size=1.4pt, line width=0.8pt] coordinates {(100,0.69000807) (127,0.66958855) (162,0.59578309) (206,0.45849149) (263,0.49283513) (335,0.40835902) (428,0.4301556) (545,0.32881116) (695,0.22318015) (885,0.22995918) (1128,0.1943189) (1438,0.13354272) (1832,0.13612578) (2335,0.12925255) (2976,0.06678346) (3792,0.07782862) (4832,0.03149077) (6158,0.0127119) (7847,0.02359309) (10000,0.02241354)};
\addplot[color={rgb,1:red,0.121569;green,0.466667;blue,0.705882}, mark=square*, mark size=1.4pt, line width=0.8pt] coordinates {(100,0.54151329) (127,0.49653188) (162,0.48673147) (206,0.373841) (263,0.39376808) (335,0.3344533) (428,0.31267499) (545,0.27266228) (695,0.16136797) (885,0.14248186) (1128,0.11115946) (1438,0.09803617) (1832,0.08067591) (2335,0.06589647) (2976,0.06415135) (3792,0.09675247) (4832,0.0660194) (6158,0.08134697) (7847,0.05773252) (10000,0.026908)};
\addplot[color={rgb,1:red,0.549020;green,0.337255;blue,0.294118}, mark=star, mark size=2.6pt, line width=1.2pt] coordinates {(100,0.50841159) (127,0.42986068) (162,0.38065757) (206,0.31338839) (263,0.31358501) (335,0.20993169) (428,0.17737456) (545,0.15161148) (695,0.15775165) (885,0.13436131) (1128,0.08612392) (1438,0.08250095) (1832,0.08052107) (2335,0.09666642) (2976,0.04482851) (3792,0.0811639) (4832,0.05004847) (6158,0.10546898) (7847,0.08415196) (10000,0.05084052)};
\nextgroupplot[
    title={F1 Score},
    ylabel={F1 Score},
]
\addplot[color={rgb,1:red,1.000000;green,0.498039;blue,0.054902}, mark=square*, mark size=1.4pt, line width=0.8pt] coordinates {(100,0.66023636) (127,0.67276) (162,0.72579079) (206,0.79973558) (263,0.78321229) (335,0.82638338) (428,0.82530965) (545,0.86219691) (695,0.90435919) (885,0.90959109) (1128,0.90921679) (1438,0.9509244) (1832,0.94469464) (2335,0.94506329) (2976,0.97876228) (3792,0.96571083) (4832,0.99265306) (6158,0.9981939) (7847,0.99761905) (10000,0.99761905)};
\addplot[color={rgb,1:red,0.121569;green,0.466667;blue,0.705882}, mark=square*, mark size=1.4pt, line width=0.8pt] coordinates {(100,0.79661687) (127,0.79546844) (162,0.81864389) (206,0.86885912) (263,0.85941814) (335,0.88829368) (428,0.90494831) (545,0.91519394) (695,0.94857865) (885,0.95828362) (1128,0.96623142) (1438,0.97365269) (1832,0.97434416) (2335,0.98650796) (2976,0.98321832) (3792,0.9752086) (4832,0.97936705) (6158,0.97984365) (7847,0.98916946) (10000,0.99761905)};
\addplot[color={rgb,1:red,0.549020;green,0.337255;blue,0.294118}, mark=star, mark size=2.6pt, line width=1.2pt] coordinates {(100,0.81236517) (127,0.83601702) (162,0.87443894) (206,0.90085295) (263,0.8933124) (335,0.94718925) (428,0.94705685) (545,0.94325302) (695,0.95454217) (885,0.95465768) (1128,0.97395222) (1438,0.97801826) (1832,0.97348869) (2335,0.97266785) (2976,0.99250794) (3792,0.97051353) (4832,0.9831865) (6158,0.96894469) (7847,0.97652066) (10000,0.98720238)};
\end{groupplot}
\node[font=\small\bfseries, rotate=90]
at ($(group c1r1.west) + (-1.7cm,0)$)
{Gaussian ratio = 0.5};
\end{tikzpicture}}
  \resizebox{0.85\linewidth}{!}{

\begin{tikzpicture}
\begin{groupplot}[
    group style={group size=3 by 1, horizontal sep=1.6cm},
    width=6.2cm, height=5.2cm,
    xmode=log, log basis x=10,
    xmin=90, xmax=11000,
    xlabel={Sample size per context},
    grid=both, grid style={line width=.1pt, draw=gray!30},
    major grid style={line width=.2pt, draw=gray!50},
    tick label style={font=\footnotesize},
    label style={font=\footnotesize},
    title style={font=\small},
    legend style={font=\scriptsize, draw=none, fill=none},
    legend cell align=left,
]
\nextgroupplot[
    title={Runtime (sec)},
    ylabel={Runtime (sec)},
    ymode=log, log basis y=10,
    legend to name=sharedlegend,
    legend columns=3,
    ytick={0.01, 0.1, 1},
    minor y tick num=0,
]
\addplot[color={rgb,1:red,0.549020;green,0.337255;blue,0.294118}, mark=star, mark size=2.6pt, line width=1.2pt] coordinates {(100,0.1162363) (127,0.11662441) (162,0.1204424) (206,0.12181733) (263,0.12409027) (335,0.18908348) (428,0.12532243) (545,0.12683086) (695,0.12865237) (885,0.13113424) (1128,0.13292009) (1438,0.13654277) (1832,0.1412215) (2335,0.1462372) (2976,0.15212846) (3792,0.16005136) (4832,0.16823649) (6158,0.18098603) (7847,0.19590498) (10000,0.21744073)};
\addlegendentry{Ours (TSCD)}
\addplot[color={rgb,1:red,0.121569;green,0.466667;blue,0.705882}, mark=square*, mark size=1.4pt, line width=0.8pt] coordinates {(100,0.21824141) (127,0.21949846) (162,0.2230116) (206,0.2246126) (263,0.22707156) (335,0.22737451) (428,0.22771678) (545,0.22898456) (695,0.23063388) (885,0.23274346) (1128,0.23566575) (1438,0.23839784) (1832,0.24298358) (2335,0.24841424) (2976,0.25387321) (3792,0.26190608) (4832,0.27003297) (6158,0.28270046) (7847,0.29775315) (10000,0.31964992)};
\addlegendentry{Pairwise correlations only}
\addplot[color={rgb,1:red,1.000000;green,0.498039;blue,0.054902}, mark=square*, mark size=1.4pt, line width=0.8pt] coordinates {(100,0.09553762) (127,0.09643061) (162,0.10112749) (206,0.10250244) (263,0.1047251) (335,0.10572015) (428,0.10574114) (545,0.1076667) (695,0.10866661) (885,0.1112847) (1128,0.11357651) (1438,0.11669557) (1832,0.12056994) (2335,0.12644298) (2976,0.13212717) (3792,0.14002692) (4832,0.14774645) (6158,0.16059157) (7847,0.17574502) (10000,0.196006)};
\addlegendentry{Proj. norm only}
\nextgroupplot[
    title={Relative Frob.\ Error},
    ylabel={Relative Frob.\ Error},
]
\addplot[color={rgb,1:red,1.000000;green,0.498039;blue,0.054902}, mark=square*, mark size=1.4pt, line width=0.8pt] coordinates {(100,0.65616063) (127,0.68156972) (162,0.64357952) (206,0.53484774) (263,0.49024616) (335,0.41991958) (428,0.3687963) (545,0.29681278) (695,0.18419238) (885,0.25188444) (1128,0.15678689) (1438,0.12960752) (1832,0.10016376) (2335,0.0878805) (2976,0.07438428) (3792,0.09654172) (4832,0.05381591) (6158,0.0495389) (7847,0.03462341) (10000,0.02305878)};
\addplot[color={rgb,1:red,0.121569;green,0.466667;blue,0.705882}, mark=square*, mark size=1.4pt, line width=0.8pt] coordinates {(100,0.52641654) (127,0.53164985) (162,0.46765749) (206,0.39373408) (263,0.38113477) (335,0.31249053) (428,0.31704387) (545,0.219562) (695,0.15822064) (885,0.16370991) (1128,0.11467359) (1438,0.06949819) (1832,0.0929828) (2335,0.10331123) (2976,0.05694488) (3792,0.0489716) (4832,0.05121765) (6158,0.05563268) (7847,0.04015488) (10000,0.06837534)};
\addplot[color={rgb,1:red,0.549020;green,0.337255;blue,0.294118}, mark=star, mark size=2.6pt, line width=1.2pt] coordinates {(100,0.48361769) (127,0.45056849) (162,0.38030441) (206,0.31416387) (263,0.30785653) (335,0.22872176) (428,0.21746268) (545,0.16803233) (695,0.10681034) (885,0.12974764) (1128,0.09016296) (1438,0.10864269) (1832,0.11678765) (2335,0.12440654) (2976,0.06988044) (3792,0.07785908) (4832,0.08175532) (6158,0.06904602) (7847,0.05528345) (10000,0.05582976)};
\nextgroupplot[
    title={F1 Score},
    ylabel={F1 Score},
]
\addplot[color={rgb,1:red,1.000000;green,0.498039;blue,0.054902}, mark=square*, mark size=1.4pt, line width=0.8pt] coordinates {(100,0.6797302) (127,0.6706251) (162,0.67830799) (206,0.75155037) (263,0.78364955) (335,0.83493452) (428,0.85116392) (545,0.87148375) (695,0.93098992) (885,0.88897547) (1128,0.94268326) (1438,0.94576142) (1832,0.96554424) (2335,0.97273928) (2976,0.96099262) (3792,0.9555222) (4832,0.9792406) (6158,0.98119664) (7847,0.9847619) (10000,0.99761905)};
\addplot[color={rgb,1:red,0.121569;green,0.466667;blue,0.705882}, mark=square*, mark size=1.4pt, line width=0.8pt] coordinates {(100,0.78426631) (127,0.78581093) (162,0.82281855) (206,0.86290652) (263,0.85153176) (335,0.9032955) (428,0.88917996) (545,0.93803009) (695,0.95570719) (885,0.94537525) (1128,0.96742668) (1438,0.98109394) (1832,0.967868) (2335,0.98098485) (2976,0.98187947) (3792,0.98454342) (4832,0.98349272) (6158,0.98374385) (7847,0.99037741) (10000,0.98423918)};
\addplot[color={rgb,1:red,0.549020;green,0.337255;blue,0.294118}, mark=star, mark size=2.6pt, line width=1.2pt] coordinates {(100,0.8193358) (127,0.83647352) (162,0.85938325) (206,0.90046896) (263,0.89365578) (335,0.92993216) (428,0.9209311) (545,0.95110353) (695,0.96554811) (885,0.95403696) (1128,0.96720125) (1438,0.96268325) (1832,0.96168471) (2335,0.96299362) (2976,0.97472689) (3792,0.96893768) (4832,0.96835257) (6158,0.9737511) (7847,0.98097729) (10000,0.98131387)};
\end{groupplot}
\node[font=\small\bfseries, rotate=90]
at ($(group c1r1.west) + (-1.7cm,0)$)
{Gaussian ratio = 1};
\node[anchor=north, yshift=-0.8cm] at (group c2r1.south) {\pgfplotslegendfromname{sharedlegend}};
\end{tikzpicture}}
  \caption{TSCD performance comparison with ablated versions across different noise settings, with Gaussian ratios 0, 0.5, and 1. Gaussian ratio denotes the ratio of nodes given Gaussian distributed noise, with the rest given Student-$t_5$ distributions. }\label{fig:ablation}
\end{figure}

\end{document}